\documentclass[11pt]{article} 

\usepackage{graphicx} 
\usepackage{natbib}  
\usepackage{caption} 
\usepackage{algorithm}
\usepackage{algorithmic}

\usepackage{booktabs}

		\usepackage{authblk}
		\usepackage{amsfonts}
		\usepackage{nicefrac}
		\usepackage{microtype}      
		\usepackage{xcolor}         
		\usepackage{amsmath}
		\usepackage{amssymb}
		\usepackage{mathtools}
		\usepackage{amsthm}
		\usepackage{bm}
		\usepackage{enumitem}
		\usepackage{subcaption}
		\usepackage{multirow}
		\usepackage[T1]{fontenc}
		\usepackage[utf8]{inputenc}

		\graphicspath{{./images/}}
		
		\theoremstyle{plain}
		\newtheorem{theorem}{Theorem}[section]
		
		\newtheorem{lemma}[theorem]{Lemma}
		
		\theoremstyle{definition}
		
		\newtheorem{assumption}[theorem]{Assumption}

\begin{document}
			
			\title{\bf BREAD: Baseline-Referenced Explanations for Anomaly Diagnosis}
			\author{Jiaqi Qiu$^{a}$\thanks{Jiaqi Qiu is with Department of Business Analytics, Amsterdam Business School, University of Amsterdam, Amsterdam, The Netherlands (e-mail: j.qiu@uva.nl).}, Rob Goedhart$^a$, Jannis Kurtz$^a$ and Inez M. Zwetsloot$^a$ \\
				$^a$ Department of Business Analytics, Amsterdam Business School,\\ University of Amsterdam, Amsterdam, The Netherlands
			}

			\date{}
			\maketitle
			
			\begin{abstract}
				Artificial Intelligence (AI)-based prospective anomaly detection methods are increasingly deployed in high-dimensional and nonlinear settings. Among these approaches, AI-based statistical process monitoring (SPM) is widely used, providing a structured framework for prospective monitoring. 
				Once an anomaly is detected, a diagnosis method is needed to identify the features driving the flagged observation away from normal behaviour. Traditional SPM diagnosis methods are typically designed for specific detection models and cannot be directly applied to AI-based methods. Model-agnostic explainable AI (XAI) offers a general framework for feature relevance explanation. However, existing methods suffer from scalability limitations or assign relevance to noise features, reducing diagnosis accuracy. 
				We propose a scalable, baseline-referenced diagnosis method that uses both the anomalous observation and normal baseline information. 
				We provide mathematical guarantees that under a mean-shift anomaly setting, the proposed method achieves higher faithfulness in detecting the features causing the anomaly compared to LIME. Simulation studies and a real-world case study validate the effectiveness of the proposed method and show that it generates more faithful and accurate diagnosis results for AI-based prospective anomaly detection methods.
			\end{abstract}

			\section{Introduction}
			Prospective anomaly detection in sequential data is essential for identifying abnormal process behaviour before it leads to system failures. Statistical Process Monitoring (SPM) provides a framework for this problem using control charts \citep{zwetsloot2023monitoring}. It has been widely used in industry \citep{colosimo2024statistical}, services \citep{tsung2008statistical}, and healthcare \citep{woodall2006use}. More recently, AI-based SPM methods have been developed for complex industrial systems \citep{lee2019process, cacciarelli2022novel, maged2024variational, qiu2025lstm}. However, in practice, detecting an anomaly is only the first step. Practitioners also need diagnosis methods that explain how much each feature contributes to the detected signal, so that appropriate actions can be taken.
			
			Explainable AI (XAI) provides tools for post-hoc explanation of model behaviour \cite{adadi2018peeking, arrieta2020explainable, dwivedi2023explainable}. In prospective anomaly detection, especially during online monitoring, the diagnosis task is to quantify how much each feature contributed to a detected signal. This requires methods that are computationally efficient and local. Moreover, explanations need to be model-agnostic, as different black-box models may be used for monitoring. These requirements make LIME \cite{ribeiro2016should} an appealing choice.
			
			However, existing LIME techniques are not optimised for signal diagnosis. Prior work has noted general limitations of LIME that are essential to signal diagnosis, including instability and limited fidelity \cite{zafar2019dlime, zhang2019should, knab2025lime}. More fundamentally, for the signal diagnosis setting, these methods explain a model output in the local neighbourhood of the queried instance. For anomalies lying far from the usual behaviour, local perturbations may remain largely within abnormal regions of the feature space. Consequently, these explanations are less informative for signal diagnosis, as they only describe the model behaviour locally around the anomaly and explain abnormal variation. In contrast, the goal of signal explanation is to explain how the variables of an identified signal drive the process away from normal operation and into an anomalous state. Thus, signal diagnosis requires explaining a local instance from a global reference perspective, presenting the normal behaviour.
			
			\begin{figure*}[h]
				\centering
				\includegraphics[width=\linewidth]{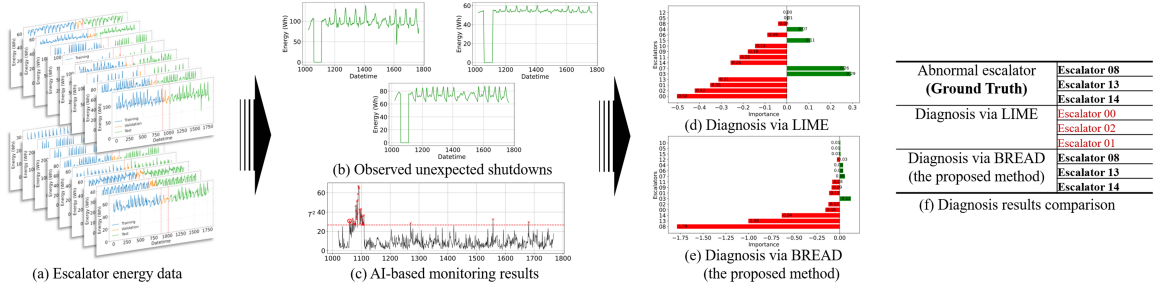}
				\caption{Real case study from escalator monitoring using an AI-based prospective anomaly detection method. 
					(a) Minute-wise escalator energy usage data from multiple escalators. 
					(b) Observed unexpected shutdowns.
					(c) VAE-LSTM-based $T^2$ control chart \citep{maged2024variational}, where anomalies align with a period of observed unexpected escalator shutdowns.
					(d) Ordered LIME diagnosis.
					(e) Ordered diagnosis by BREAD.
					(f) Comparison of results, showing that the diagnosis by BREAD matches the ground truth.} 
				\label{fig:framework}
			\end{figure*}
			
			To highlight the limitations of LIME, we introduce a real monitoring scenario. We consider an escalator monitoring system \citep{zwetsloot2023remaining}, as shown in Figure~\ref{fig:framework}(a). Monitoring the energy data in real time allows engineers to assess the operational status of the escalator system. A possible anomaly, for instance, an unexpected shutdown period depicted in Figure~\ref{fig:framework}(b), is expected to be detected by the monitoring model. Figure~\ref{fig:framework}(c) shows that the applied monitoring method detects anomalies during the shutdown period. However, only detection is not enough for practitioners, as they also need to identify the escalators that caused the whole system to become anomalous. The ordered diagnosis results are illustrated in Figure~\ref{fig:framework}(d--e). As depicted in Figure~\ref{fig:framework}(f), LIME fails to assign high relevance to the observed shutdown escalators. Such a misdiagnosis may direct practitioners to take inappropriate actions. Thus, this paper addresses the challenge of identifying the variables that actually drive the system to abnormal conditions under real-time computational constraints.
			
			To address this, we propose a scalable model-agnostic explanation method, Baseline-Referenced Explanations for Anomaly Diagnosis (BREAD), for AI-based prospective anomaly detection that incorporates a reference point representing the in-control baseline information. In essence, BREAD modifies the LIME algorithm by introducing a novel sampling scheme and a sampling-aligned weighting kernel. The main idea is to fit a regression model on the weighted sampled data around both the anomalous observation and a reference normal point. By introducing the baseline information into the surrogate construction, the resulting explanations become sparser, concentrating relevance on anomaly related variables while assigning near-zero relevance to irrelevant features. In prospective anomaly detection, anomalous events are often caused by a limited subset of variables. This concentration can improve the faithfulness of the explanations because such explanations are better aligned with the variables that actually drive the anomalous signal. Meanwhile, the computational cost of BREAD is measured by the number of synthetic samples and the cost of the black-box model, and thus, BREAD remains computationally efficient. 
			
			For our method, we provide a theoretical analysis showing that, under certain assumptions that are both realistic and compatible with common monitoring settings, the proposed approach achieves improved faithfulness for signal diagnosis relative to LIME. The proposed method is evaluated both in simulation studies and a real case study. The simulations show the faithfulness, robustness of BREAD, and stability against sampling noise. The case study revisits the motivating example in Figure~\ref{fig:framework} from a deployment perspective, providing an empirical evaluation of diagnosis alignment with observed shutdown events and computational efficiency.
			
			Thus, we summarise our main contributions as follows:
			\begin{itemize}[]
				\item We propose BREAD, a scalable, model-agnostic explanation method that incorporates baseline normal information through a novel sampling scheme and a modified weighting kernel.
				\item We provide theoretical proofs and analysis showing that, given an identified anomaly, BREAD yields asymptotically sparse explanations, which in turn improve alignment with the ground truth under the stated assumptions.
				\item We validate the theoretical results through simulation studies and a real-world case study, demonstrating the effectiveness of the proposed method for black-box-based prospective anomaly detection.
			\end{itemize}
			
			The remainder of the paper is organized as follows. Section 2 reviews related work. Section 3 formalizes the problem, presents BREAD, and establishes its theoretical properties. Section 4 evaluates BREAD  through simulation studies and the real-world case study. Section 5 concludes the paper.
			
			\section{Related Work}
			A major family of methods for prospective anomaly detection in sequential process data is statistical process monitoring (SPM), which uses control charts to detect departures from normal process behaviour \cite{zwetsloot2023monitoring}. Early work focused on decomposing alarms from traditional control charts, especially Hotelling's $T^2$ chart under Gaussian assumptions, to identify the variables most responsible for an out-of-control (OOC) signal \cite{mason1995decomposition, li2008causation, kim2016adaptive}. Further diagnosis methods have been developed beyond the Gaussian setting, for example, multivariate binomial processes \cite{hou2023simple}. These methods are effective in finding a subset of variables driving a signal, but are tailored to specific chart structures, monitoring statistics, and distributional assumptions. By contrast, AI-based SPM methods typically apply complex embeddings to capture underlying process behaviour and extract representations from a latent space \citep{yu2019deep, yu2019robust, zhang2019gaussian, lee2019process, cacciarelli2022novel, maged2024variational}. Such methods have been shown to capture latent behaviour for process monitoring better \cite{cacciarelli2023hidden}. However, the latent-space representations are often difficult to translate into actionable information in the original space. 
			
			This motivates the use of explainable AI (XAI) methods that quantify feature relevance for black-box models. A large body of work has proposed feature-relevance methods such as saliency maps \cite{simonyan2013deep}, Layer-wise Relevance Propagation (LRP) \cite{bach2015pixel}, Grad-CAM \cite{selvaraju2016grad}, SHAP-based approaches \cite{lundberg2017unified}, and Integrated Gradients (IG) \cite{sundararajan2017axiomatic}. In SPM, prior work has applied feature-relevance XAI methods, including LRP \cite{agarwal2021explainability}, IG \cite{sipple2020interpretable, cacciarelli2022novel, bhakte2022explainable}, SHAP \cite{kim2021explainable,hwang2021sfd,choi2022explainable, baek2023failure, jang2023explainable, szelkazek2024semantic}. However, structured performance evaluation of these methods in SPM remains limited. In addition, these methods are either computationally expensive or model-specific. For model-specific methods, assumptions about gradients, architecture, or internal model structure limit their applicability \cite{adadi2018peeking}. In SPM settings, where the monitoring model may vary, model-agnostic explanation methods are therefore especially attractive. 
			
			Thus, among these feature-relevance methods, LIME \cite{ribeiro2016should} is particularly appealing because it is local, model-agnostic, and computationally efficient \cite{mishra2017local, rabold2018explaining, zafar2019dlime, otto2025coherent, jia2025limefldl}, and has been applied to alarm explanation in SPM by \citet{bhakte2023alarm}. However, prior work has also highlighted important limitations of LIME, including limited fidelity, instability due to stochastic neighbourhood sampling, and difficulties in domain-specific adaptation \cite{zafar2019dlime, zhang2019should, zhou2021s, tan2023glime,knab2025lime}. These issues are particularly critical in online SPM, where explanations must be both reliable and computationally efficient. Moreover, the diagnosis objective is not just to explain the model output at a single point, but to identify the variables that drive the process away from the in-control region.

			\section{Proposed Method}
			
			\subsection{Problem Description} \label{sec:ProblemDescr}
			Let $f: \mathbb{R}^{d} \rightarrow \mathbb{R}$ denote an AI-based anomaly detection model that maps an observation $\bm{t} \in \mathbb{R}^{d}$ to an anomaly score $s \in \mathbb{R}$, i.e., $s = f(\bm{t})$.
			A decision function $\theta_s:\mathbb{R}\rightarrow\{0,1\}$ maps the anomaly score $s$ to a binary label, where $0$ denotes normal behaviour and $1$ denotes anomalous behaviour. The corresponding normal region is defined as
			\(
			\mathcal{A}=\{\bm{t}\in\mathbb{R}^{d}:\theta_s(f(\bm{t}))=0\}.
			\)
			For example, in the SPM context, $s$ is often assumed to follow a certain distribution and a control limit is then set as a guideline of the distribution.
			
			During online monitoring, new data are denoted by $\bm{t}_j \in \mathbb{R}^{d}, j > N$. An observation at time $j^*$ is identified as an anomaly if $\bm{t}_{j^*} \notin \mathcal{A}$. When an anomaly is detected, the diagnosis method is expected to provide an explanation indicating how much each feature contributes to the signal.
			
			\subsection{Framework} \label{sec:framew}
			Recall that our proposed method BREAD approximates the behaviour of $f$ with respect to $\bm{t}_{j^*}$ and the reference normal point via a linear surrogate model whose coefficients $\bm{\beta}$ quantify the feature contribution from the reference normal point to the anomalous observation. 
			
			In practice, we estimate $\bm{\beta}$ by fitting a weighted ridge regression surrogate to synthetic samples drawn around both $\bm{t}_{j^*}$ and $\bm{t}_{\mathrm{ref}}$, where $\bm{t}_{\mathrm{ref}} \in \mathcal{A}$ denotes the reference normal point that represents the normal operating conditions. It is necessary to verify if $\bm{t}_{\mathrm{ref}} \in \mathcal{A}$ when choosing the mean of the normal training data as $\bm{t}_{\mathrm{ref}}$, as it is not guaranteed for nonlinear or non-convex normal regions.
			After estimating the coefficients $\hat{\bm{\beta}}$, we report the top $K$ features according to $|\hat{{\beta}}_i|$ as the explanation for the detected signal.

			Synthetic samples are generated by adding random perturbations from a distribution, e.g., $\mathcal{N}_d (0, \Sigma)$, to both $\bm{t}_{j^*}$ and $\bm{t}_{\mathrm{ref}}$. 
			Through this generation procedure, we obtain the sampled set $\mathcal{Z} = \{ \bm{z}_i \}^n_{i=1}$ with $n$ sampled points and the corresponding matrix $Z \in \mathbb{R}^{n \times d}$, whose $i$-th row is $\bm{z}_i^\intercal$.
			For notational convenience, assume that $n$ is even, with the first $\frac{n}{2}$ points sampled around $\bm{t}_{j^*}$ and the remaining $\frac{n}{2}$ points sampled around $\bm{t}_{\mathrm{ref}}$.
			We then obtain the corresponding outputs $y_i=f(\bm{z}_i)$, and define $\bm{y} = (y_1, \ldots, y_n)^{\intercal} \in \mathbb{R}^n$. 
			
			After sampling, each sampled point $\bm{z}_{i}$ is weighted by a Gaussian-type kernel,
			\begin{equation}\label{eq:weight}
				\pi_i = \pi_{\bm{t}_{j^*},\bm{t}_{\mathrm{ref}}} ( \bm{z}_i ) = \exp \left( - \frac{d^2_{\bm{t}_{j^*},\bm{t}_{\mathrm{ref}}}(\bm{z}_i)}{\sigma^2} \right),
			\end{equation}
			where $\sigma$ is the kernel width controlling the weights, and $d_{\bm{t}_{j^*},\bm{t}_{\mathrm{ref}}}$ denotes the distance of the sampled data with respect to the anomalous observation $\bm{t}_{j^*}$ and the reference point $\bm{t}_{\mathrm{ref}}$.
			Specifically, the distance is defined as
			\begin{equation}\label{eq:dist}
				d_{ \bm{t}_{j^*}, \bm{t}_{\mathrm{ref}} } (\bm{z}_i) = \sqrt{{\| \bm{z}_i - \bm{t}_{j^*} \|}_2 \cdot {\| \bm{z}_i - \bm{t}_{\mathrm{ref}} \|}_2}.
			\end{equation}
			This distance assigns large weights to samples that are close to either the detected anomaly or the reference normal point. As a boundary case, when $\bm{t}_{j^*}=\bm{t}_{\mathrm{ref}} $, \eqref{eq:weight} and \eqref{eq:dist} imply that the kernel used in BREAD becomes a one-point weighting, which is analogous to the weighting used in vanilla LIME. Figure~\ref{fig:weightheatmap} compares the weight distribution of both kernels.
			
			\begin{figure}[!h]
				\centering
				\begin{subfigure}{0.48\columnwidth}
					\centering
					\includegraphics[width=\columnwidth]{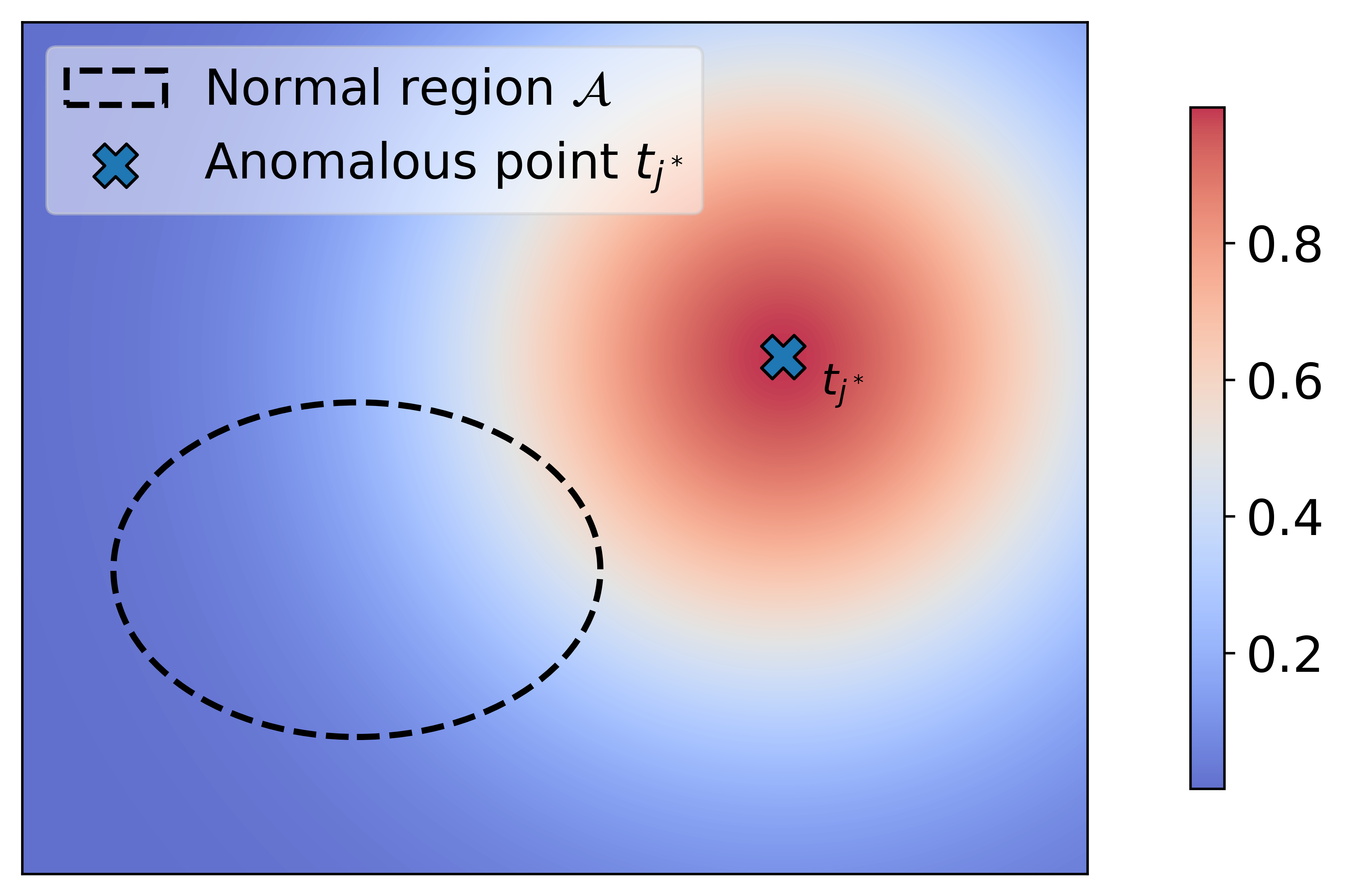}
					\caption{}
					\label{fig:withoutref}
				\end{subfigure}
				\hfil
				\begin{subfigure}{0.48\columnwidth}
					\centering
					\includegraphics[width=\columnwidth]{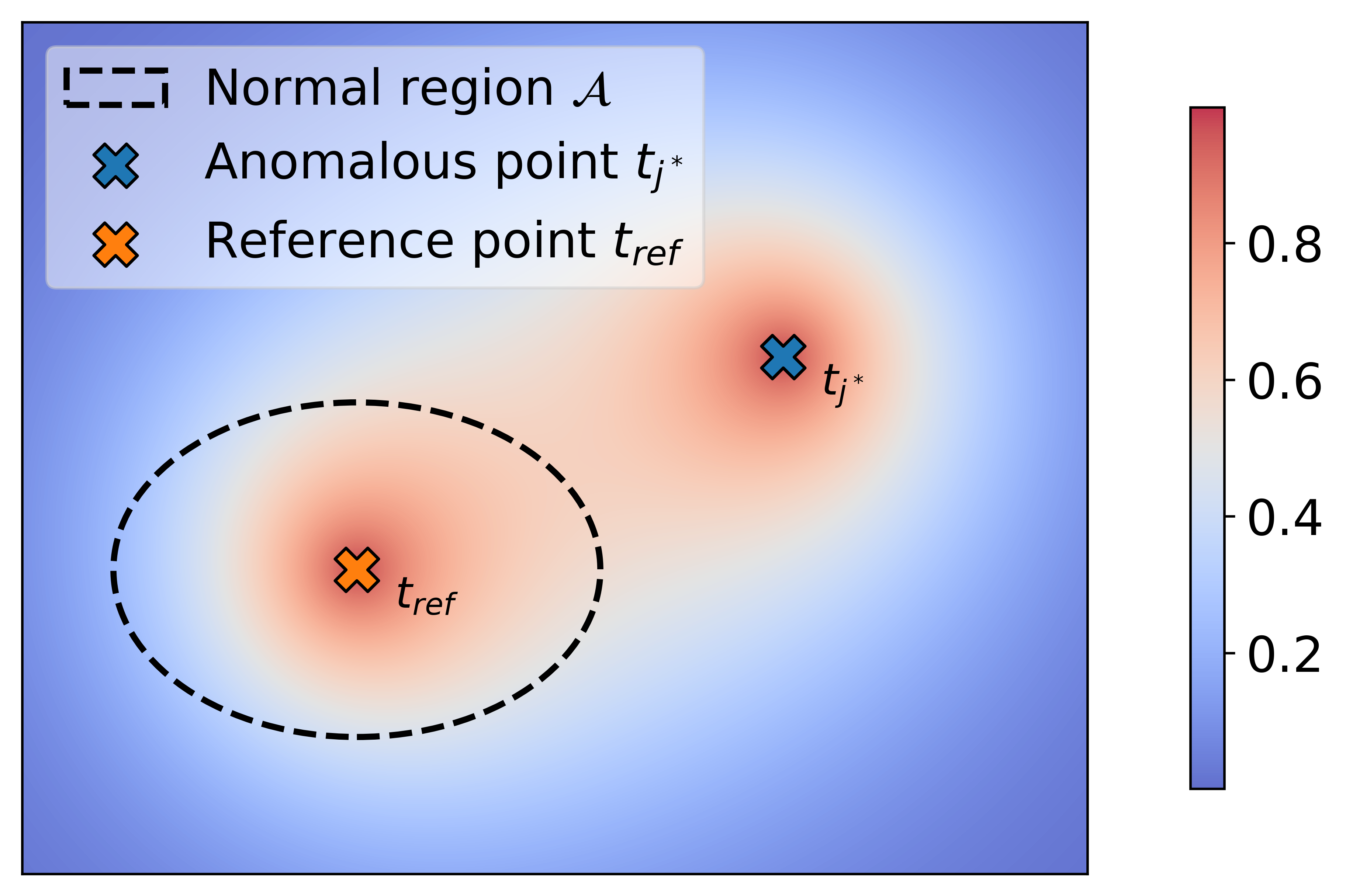}
					\caption{}
					\label{fig:withref}
				\end{subfigure}
				\caption{
					Heat maps of weights under the same kernel width $\sigma$.
					(a) Weighting using the anomalous point only. Weights are higher when closer to $\bm{t}_{j^*}$.
					(b) Weighting using both the reference point and the anomalous point. Weights are higher when either near $\bm{t}_{\mathrm{ref}}$ or $\bm{t}_{j^*}$.
				}
				\label{fig:weightheatmap}
			\end{figure}
			
			We estimate $\hat{\bm{\beta}}$ by the following weighted loss function, 
			\begin{equation*}
				\begin{aligned}
					L( f, \bm{\beta},\beta_0,  \Pi ) & =  \sum_{i=1}^n  \pi_{i}( \bm{z}_i ) \left({\bm{\beta}}^{\intercal} \bm{z}_i + {\beta_0} - f ( \bm{z}_i ) \right) ^{2}\\
					& = ( Z {\bm{\beta}} +{\beta_0 \bm{1}} - \bm{y})^{\intercal} \Pi ( Z {\bm{\beta}} + {\beta_0 \bm{1}}  - \bm{y}),
				\end{aligned}			
			\end{equation*}
			where $ \Pi = \text{diag} ( \pi_{1},\ldots \pi_n)$ denotes the diagonal matrix of sample weights for each $\bm{z}_i$, $\beta_{0}$ is the intercept term, $\bm{1}$ is a vector of ones with length $n$, and $\bm{y} \in \mathbb{R}^{n}$ stacks the monitoring model outputs $y_i =  f( \bm{z}_i )$.
			Using the weighted sampled data, we fit a regression model by minimising the regularised loss,
			\begin{equation}\label{eq:optprob}
				\min_{ \bm{ \beta} \in \mathbb{R}^{d} , \, \beta_{0} \in \mathbb{R} }  \, \frac{1}{2} ( Z {\bm{\beta}} +{\beta_0 \bm{1}} - \bm{y})^{\intercal} \Pi ( Z {\bm{\beta}} + {\beta_0 \bm{1}} - \bm{y})
				+ \frac{\lambda}{2}  {\bm{\beta}}^{ \intercal} \bm{I}_{d} {\bm{\beta}},
			\end{equation}
			where $\frac{1}{2}$ is included to simplify the derivatives, $\lambda$ is a positive penalty weight that shrinks the coefficients toward zero, and $\bm{I}_{d} \in \mathbb{R}^{ d \times d }$ is an identity matrix. The weighted ridge surrogate linear model has the analytical solution (see Supplementary Appendix~B).
			\begin{equation}\label{eq:estimator}
				\hat{\bm{\beta}} = (A + \lambda \bm{I}_{d} )^{-1} \bm{b},
			\end{equation}
			where
			$A =  \left(Z^{ \intercal} \Pi  Z - \frac{Z^{ \intercal} \Pi \bm{1} \left(Z^{ \intercal} \Pi \bm{1}\right)^{\intercal}}{\sum_{i=1}^{n} \pi_i}\right),$ and 
			$\bm{b} = Z^{\intercal} \Pi \bm{y} -\frac{1}{\sum_{i=1}^{n} \pi_i}\left(Z^{ \intercal} \Pi \bm{1}\right)\left(\bm{1}^{\intercal} \Pi \bm{y}\right)$.
			From \eqref{eq:optprob}, since $\lambda>0$ and $\pi_i>0$, the weighted ridge objective is strictly convex. Therefore, the solution exists and is unique. This property ensures that the proposed surrogate explanation is well defined.
			The diagnosis for the identified anomaly $\bm{t}_{j^*}$ is then given by the indices of the $K$ largest components of $|\hat{\bm\beta}|$.

			\subsection{Analytical Properties}\label{sec:anaprop}
			This section discusses the analytical properties of BREAD, focusing on (i) asymptotic sparsity when an anomaly occurs in the feature space, and (ii) stability to perturbations.
			
			We next analyse whether the surrogate coefficients concentrate on the truly shifted features as the anomaly becomes farther from the normal reference point.
			Specifically, let the normal reference point be $\bm{t}_{\mathrm{ref}} \in \mathbb{R}^{d}$, and suppose an anomaly $\bm{t}_{j^*}$ is detected. Denote the shift magnitude $\delta = \|\bm{t}_{j^*} - \bm{t}_{\mathrm{ref}}\|_2 > 0$ and the direction $\bm{u} = (\bm{t}_{j^*} - \bm{t}_{\mathrm{ref}})/\delta$, yielding
			\begin{equation}\label{eq:meanshf}
				\bm{t}_{j^*} = \bm{t}_{\mathrm{ref}} + \delta \bm{u}.
			\end{equation}
			Let $\mathcal{K} = \{\ell : u_\ell \neq 0\}$ denote the set of shifted features and $\mathcal{R}=\{1,\dots,d\}\setminus\mathcal{K}$, where $|\mathcal{K}| < d$.
			
			As described above, our method generates $\mathcal{Z}^{(\delta)}=\{\bm{z}_i^{(\delta)}\}_{i=1}^n$ and index sets $I_0^{(\delta)}$ and $I_1^{(\delta)}$ with $|I_0^{(\delta)}|=|I_1^{(\delta)}|=n/2$ containing samples around $\bm{t}_{\mathrm{ref}}$ and $\bm{t}_{j^*}$, respectively. 
			The following two assumptions formalise the sampling design and the output structure. 
			\begin{assumption}[Geometric sampled data assumption]\label{ass:samdata}
				There exist constants $c,\varepsilon,\rho>0$ (independent of $\delta$) and core index sets $J^{(\delta)}_0\subset I^{(\delta)}_0$ and $J^{(\delta)}_1\subset I^{(\delta)}_1$ with $|J^{(\delta)}_0|=|J^{(\delta)}_1|\geq 1$, such that 
				$ \| \bm{z}_i^{(\delta)}-\bm{t}_{\mathrm{ref}} \|_2 \le \frac{c}{\delta^2} $, for $i\in J^{(\delta)}_0,$ and $ \| \bm{z}_i^{(\delta)}-\bm{t}_{j^*} \|_2 \le \frac{c}{\delta^2} $, for $i\in J^{(\delta)}_1.$
				Furthermore, for the non-core points, 
				$ \varepsilon \le \|\bm{z}_i^{(\delta)}-\bm{t}_{\mathrm{ref}} \|_2 \le \rho$, for $i\in I^{(\delta)}_0\setminus J^{(\delta)}_0,$ and $\varepsilon \le \|\bm{z}_i^{(\delta)}-\bm{t}_{j^*} \|_2 \le \rho$, for $i\in I^{(\delta)}_1\setminus J^{(\delta)}_1.$
			\end{assumption}
			
			Assumption~\ref{ass:samdata} is consistent with a local perturbation scheme and guarantees at least one sampled point in each group has non-trivial weights. We also assume that $n$ is fixed.
			
			Next, we impose a structural assumption on the corresponding outputs. Write these outputs as $y_i^{(\delta)} = f(\bm{z}_i^{(\delta)})$, and define $\bar{y}_s^{(\delta)}$ as the group mean over $I_s^{(\delta)}$ for $s\in\{0,1\}$.
			Because the signal is caused by the shift, we assume that the black-box monitoring output has a persistent between group contrast $\Delta(\delta) := \bar{y}_1^{(\delta)} - \bar{y}_0^{(\delta)}$, while within-group fluctuations are bounded. To ensure that the dominant output change is the between-group $\Delta(\delta)$, we additionally assume that the black-box monitoring output is nearly constant on the core points near each centre.
			\begin{assumption}[Output structure]\label{ass:out}
				Assume there exist constants $B_y<\infty$, $c_\Delta, c_y >0$ independent of $\delta$ such that for each $s \in\{0,1\}$, each $i\in I_s^{(\delta)}$, $\bigl| {y}_i^{(\delta)}-\bar y_s^{(\delta)}\bigr| \le B_y,$ and $\Delta(\delta) \geq c_\Delta$.   
				Moreover, for $s \in\{0,1\}$ and $i \in J_s^{(\delta)}$,
				$|{y}_i^{(\delta)}-\bar y_s^{(\delta)} | \leq \frac{c_y}{\delta}.$
			\end{assumption}
			
			Note that Assumption~\ref{ass:out} is a monitoring-diagnosis design condition, which is similar to local Lipschitz continuity, and formalises the local output structure after a meaningful monitoring signal has been detected. Specifically, it requires the monitoring model to produce a persistent output contrast between the synthetic samples around the detected anomalous point and samples around the reference normal point.

			Assumptions~\ref{ass:samdata}--\ref{ass:out} are both practically satisfied in the monitoring context.
			Assumption~\ref{ass:samdata} is typically satisfied by selecting an appropriate sample scheme.
			For example, it can be achieved by including the two centre points $\bm{t}_{\mathrm{ref}}$ and $\bm{t}_{j^*}$ in the synthetic sample set. Then their kernel weights are equal to one. The remaining non-core samples can be generated from a truncated normal distribution around their centres.  
			Assumption~\ref{ass:out} is satisfied when a meaningful monitoring signal has been detected, as the anomalous and reference outputs show a persistent anomaly score difference. 
			Together, the two assumptions ensure that (i) each set contains at least one ``core'' point extremely close to its centre, and (ii) the anomaly score $f(\bm{t}_{j^*})$ exceeds the threshold due to the shift in $\mathcal{K}$, rather than the bias.

			Under the geometric design and output structure, we establish the following asymptotic sparsity guarantee for BREAD.
			\begin{theorem}[Asymptotic sparsity]\label{th:faith}
				Suppose the shift setup in \eqref{eq:meanshf} holds, the synthetic data matrix $Z^{(\delta)}$ satisfies Assumption~\ref{ass:samdata}, and the corresponding outputs, ${y}_i^{(\delta)} = f(\bm{z}_i^{(\delta)})$, satisfy Assumption~\ref{ass:out}.
				Let $\hat{\bm{\beta}}$ denote the estimator in \eqref{eq:estimator} and define $\epsilon  = \dfrac{\| \hat{\bm{\beta}}_{\mathcal{R}} \|_{2} }{\| \hat{\bm{ \beta }}_{ \mathcal{K} } \|_{2}}$.
				Then,
				\[{\epsilon} \rightarrow 0, \qquad \text{as} \, \delta \rightarrow \infty.\]
			\end{theorem}
			
			Theorem~\ref{th:faith} establishes that, as the anomaly moves further from the normal reference, 
			the sampling geometry and kernel weighting jointly decouple the shifted and unshifted feature blocks in the surrogate, driving relevance to concentrate on $\mathcal{K}$ and vanish on $\mathcal{R}$. The proof is provided in Appendix~C.
			
			In contrast, LIME does not enjoy this guarantee that the ratio of coefficients need not converge to zero, even as $\delta \rightarrow \infty$.
			To illustrate this comparison, consider a simple two-dimensional case with \(\bm{t}_{\mathrm{ref}}=(0,0)\) and \(\bm{t}_{j^*}=(\delta,0)\).
			In this case, the anomaly is generated only by the first feature, so \(\mathcal{K}=\{1\}\) and \(\mathcal{R}=\{2\}\).
			Suppose the black-box model is locally quadratic in the neighbourhood of $\bm{t}_{j^*}$, $f(z_1,z_2) = z_1^2 + \gamma z_1 z_2$ with $\gamma \neq 0$.
			Then $\nabla f (\delta,0) = (2\delta, \gamma\delta)^{\intercal}$. Thus, although the shift occurs only in the first feature, the local sensitivity of $f$ around $\bm{t}_{j^*}$ has a nonzero component in the second feature. A LIME explanation fitted locally around $\bm{t}_{j^*}$ may assign a nonzero coefficient to the non-shifted feature, and the coefficients ratio is approximately $\epsilon^L \approx |\gamma|/2$.
			By contrast, Theorem~\ref{th:faith} states that the proposed reference-based surrogate yields
			\(||\hat{\bm{\beta}}_{\mathcal{R}}||_2 / ||\hat{\bm{\beta}}_{\mathcal{K}}||_2 \rightarrow 0 \) as \(\delta \rightarrow \infty \).
			This shows that LIME's coefficient ratio is bounded away from zero by a constant $\epsilon^L \approx |\gamma|/2$, regardless of how large $\delta$ becomes, which is a fundamental limitation that BREAD overcomes.
			
			Finally, we study the stability of the estimator via its sensitivity to perturbations in $(A,\bm{b})$.
			\begin{lemma}[Sensitivity]\label{lm:sens}
				Consider the linear system $(A+ \lambda \bm{I}_d) \bm{\beta} = \bm{b}$, where $A$ and $\bm{b}$ are defined as in \eqref{eq:estimator}. If $\| (A+ \lambda \bm{I}_d)^{-1} \|_p \| \Delta (A+ \lambda \bm{I}_d) \|_p < 1$, then
				\begin{equation}
					\begin{aligned}
						\frac{\| \Delta \bm{\beta} \|_p}{\| \bm{\beta} \|_p}
						\le & \frac{\kappa_p((A+ \lambda \bm{I}_d))}{1 - \kappa_p((A+ \lambda \bm{I}_d))\, \dfrac{\| \Delta (A+ \lambda \bm{I}_d) \|_p}{\| (A+ \lambda \bm{I}_d) \|_p}}\\
						& \left( \frac{\| \Delta (A+ \lambda \bm{I}_d) \|_p}{\| (A+ \lambda \bm{I}_d) \|_p}+\frac{\| \Delta \bm{b} \|_p}{\| \bm{b} \|_p}\right),
					\end{aligned}
				\end{equation}
				where $\kappa_p((A+ \lambda \bm{I}_d)) = \| (A+ \lambda \bm{I}_d) \|_p \| (A+ \lambda \bm{I}_d)^{-1} \|_p$ is the ($p$-) condition number of the matrix $A+ \lambda \bm{I}_d$, while $\Delta\bm{\beta}, \Delta (A+ \lambda \bm{I}_d) $, and $\Delta \bm{b}$ denote the perturbation in $\bm{\beta}$, $A+ \lambda \bm{I}_d$, and $\bm{b}$, respectively.
			\end{lemma}
			
			Lemma~\ref{lm:sens} shows that the relative perturbation in the estimated coefficient vector is controlled by the condition number of $A+ \lambda \bm{I}_d$ and by the perturbations in $A$ and $\bm{b}$, which follows from standard perturbation theory for linear systems.
			Therefore, when the weighted ridge system is well conditioned, the generated explanation is stable with respect to small perturbations in the weighted sampled data, or model outputs.
			
			\section{Experiments}
			In this section, we evaluate BREAD using a synthetic simulation and a case study.
			\subsection{Simulation Study}
			The simulation is designed to mimic high-dimensional sequential data with temporal dependence, cross-sectional correlation, and seasonal variation. Although the monitoring model used in this study follows an SPM control chart framework, the evaluation focuses on the more general task of diagnosing detected anomalies. Specifically, we evaluate the proposed approach along three dimensions: (i) faithfulness measures alignment of diagnosis with the ground truth; (ii) robustness shows sensitivity of the explanation to observation perturbations; (iii) stability shows the consistency across resampled datasets.
			
			\subsubsection{Data Generation and Monitoring Model Setup}
			We first generate a baseline data set representing normal behaviour, drawn from a multivariate autoregressive model of order one (AR(1)) with block-structured cross-sectional correlation and an additive seasonal component. Specifically, we simulate $d=500$ features over $n=2000$ time points. The covariance structure $\Sigma_0$ is constructed to exhibit block-wise dependence (i.e., correlation blocks) to reflect clustered feature relationships. Variables are ordered by correlation blocks, yielding a clear block-diagonal pattern with strong within-block dependence ($>0.8$) and comparatively weaker between-block correlations (lighter off-diagonal regions, $<0.5$). Two kinds of seasonality with periods of $24$ and $168$ are induced by first-order Fourier components. The data generation procedure is detailed in Appendix~{D.2} of the Supplementary.
			
			To generate anomalous test data, we introduce feature shifts, defined as \eqref{eq:meanshf}, from $t=1600$. 
			These shifts define three anomaly scenarios, adapted from \citet{cacciarelli2022novel}, ranging from $-5$ to $5$ stepped by $1$, yielding 33 experimental conditions in total. The zero-shift setting is used as the reference case.
			
			\begin{itemize}
				\item \textbf{Scenario 1} \textit{Anomaly within one block}\\
				The means of all 53 features in the first correlation block are shifted by the given value.
				
				\item \textbf{Scenario 2} \textit{Anomaly across blocks}\\
				From each block, three features are randomly selected and their means are shifted, for a total of 45 features.
				
				\item \textbf{Scenario 3} \textit{Combination of Scenarios 1\&2}\\
				The means of all 53 features in the first block are shifted, and three features from each remaining block are also shifted, for a total of 95 features.
			\end{itemize}
			
			For prospective anomaly detection, we use an LSTM-based $T^2$ control chart model. The anomaly score $s$ is defined as the $T^2$ statistic computed on the LSTM residuals, while the threshold $\theta_s$ is set based on an $F$ distribution \citep{sullivan1996comparison} when given the Type I error rate $\alpha$. Furthermore, the empirical anomaly detection performance of this method is shown in Appendix~{D.3}.
			
			\subsubsection{Benchmark Methods}
			For faithfulness and robustness, we compare BREAD against vanilla LIME and a naive baseline that reduces KernelSHAP \citep{lundberg2017unified} to leave-one-out coalitions.
			Specifically, the vanilla LIME is the original LIME using the Gaussian kernel and the ridge regression as the surrogate model. Sample size, kernel bandwidth $\sigma$ and ridge penalty parameter $\lambda$ are fixed across all experiments for both LIME and BREAD. For the naive method, the reference point is selected in the same way as in BREAD. Detailed hyperparameters are discussed in Appendix~{D.4}.
			In addition, when comparing faithfulness, we introduce the expected faithfulness of detecting the features causing the error when randomly selecting $K$ features from $d$. For comparing sampling stability, we only consider sampling-based methods, i.e., BREAD and LIME.

			\subsubsection{Faithfulness}
			Faithfulness is assessed by comparing the top-$K$ features returned by each explanation method to the known ground-truth affected features (i.e., those in the shifted block) for the first detected anomalous observation per run. Specifically, we repeat the full simulation and detection pipeline across 1000 random seeds, extract the first outlier identified under each seed, and quantify agreement between the top-$K$ explanations and the ground truth by the cosine similarity $S_F$. 
			Given the set $\mathcal{K}$ of $K$ features driving the process to abnormal, and a set $D_r$ of the selected top-$K$ features, the faithfulness at run $r$ is defined as $S^{(r)}_F (\mathcal{K}, D_r) = |\mathcal{K} \cap D_r| / \sqrt{(|\mathcal{K}|\cdot|D_r|)}$. 
			
			The expected cosine similarity of a random $K$-feature selection is $K/d$ (Random Guess in Figure~\ref{fig:compar_result_lime}), following a hypergeometric distribution, which gives $\mathbb{E}(S_F) = 0.106, 0.09, 0.190$ for Scenarios 1--3, respectively.
			
			\begin{figure}[!h]
				
				\centering
				\includegraphics[width=0.55\columnwidth]{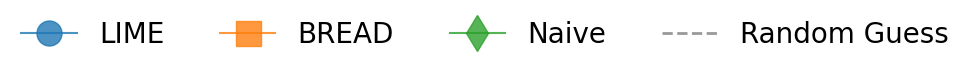}
				
				\begin{subfigure}{0.325\columnwidth}
					\centering
					\includegraphics[width=\columnwidth]{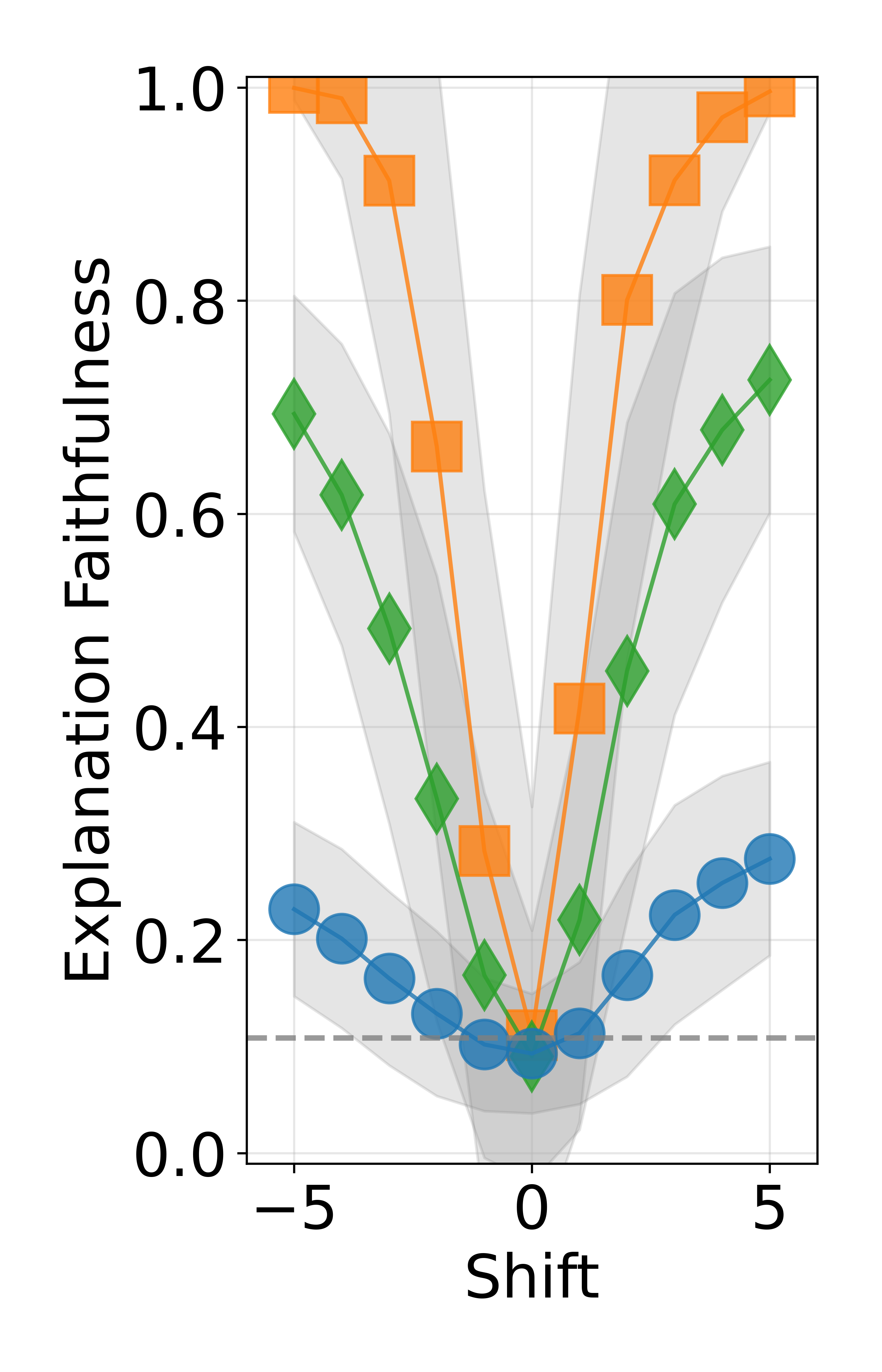}
					\caption{Scenario 1}
				\end{subfigure}
				\begin{subfigure}{0.325\columnwidth}
					\centering
					\includegraphics[width=\columnwidth]{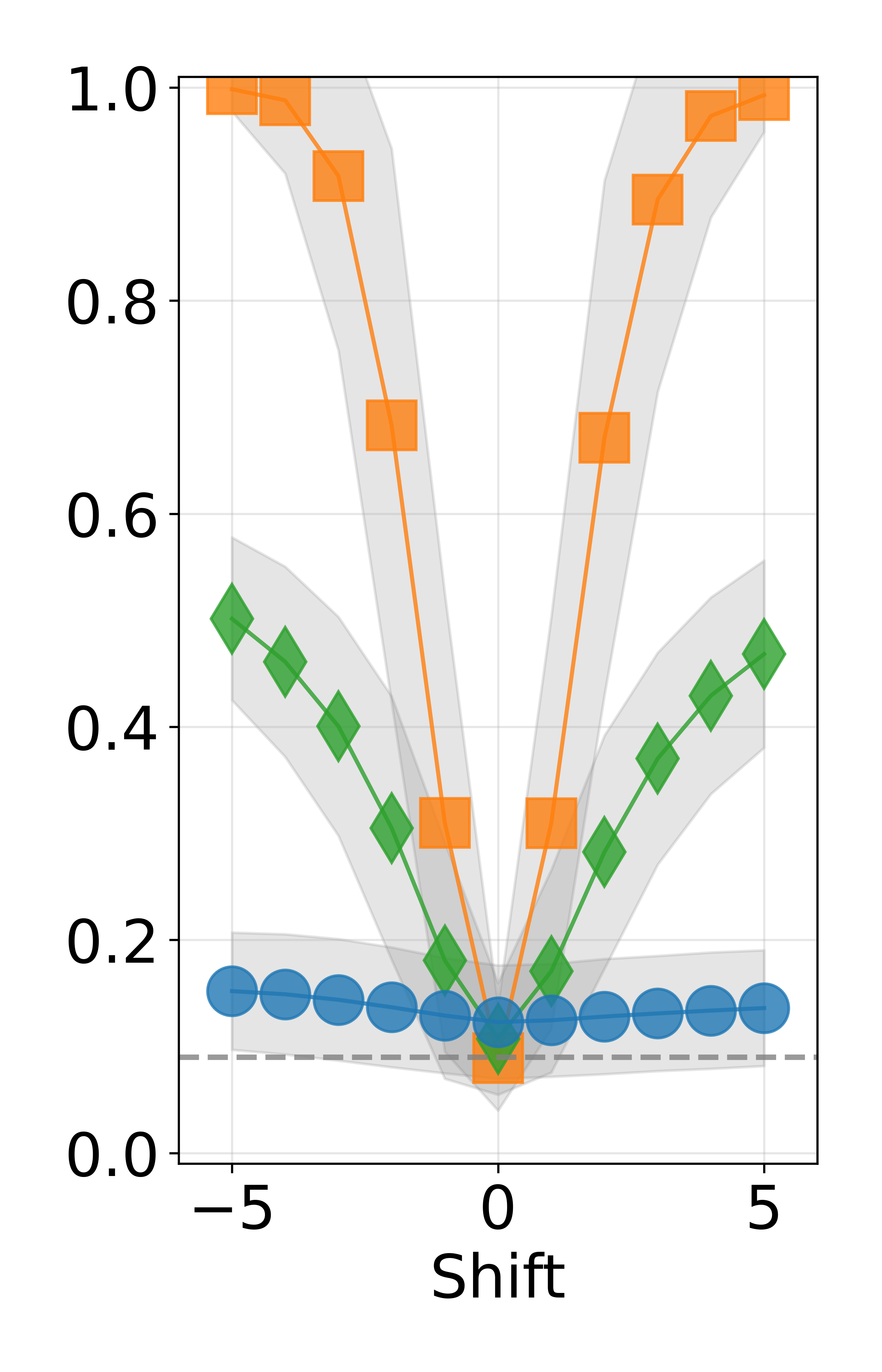}
					\caption{Scenario 2}
				\end{subfigure}
				\begin{subfigure}{0.325\columnwidth}
					\centering
					\includegraphics[width=\columnwidth]{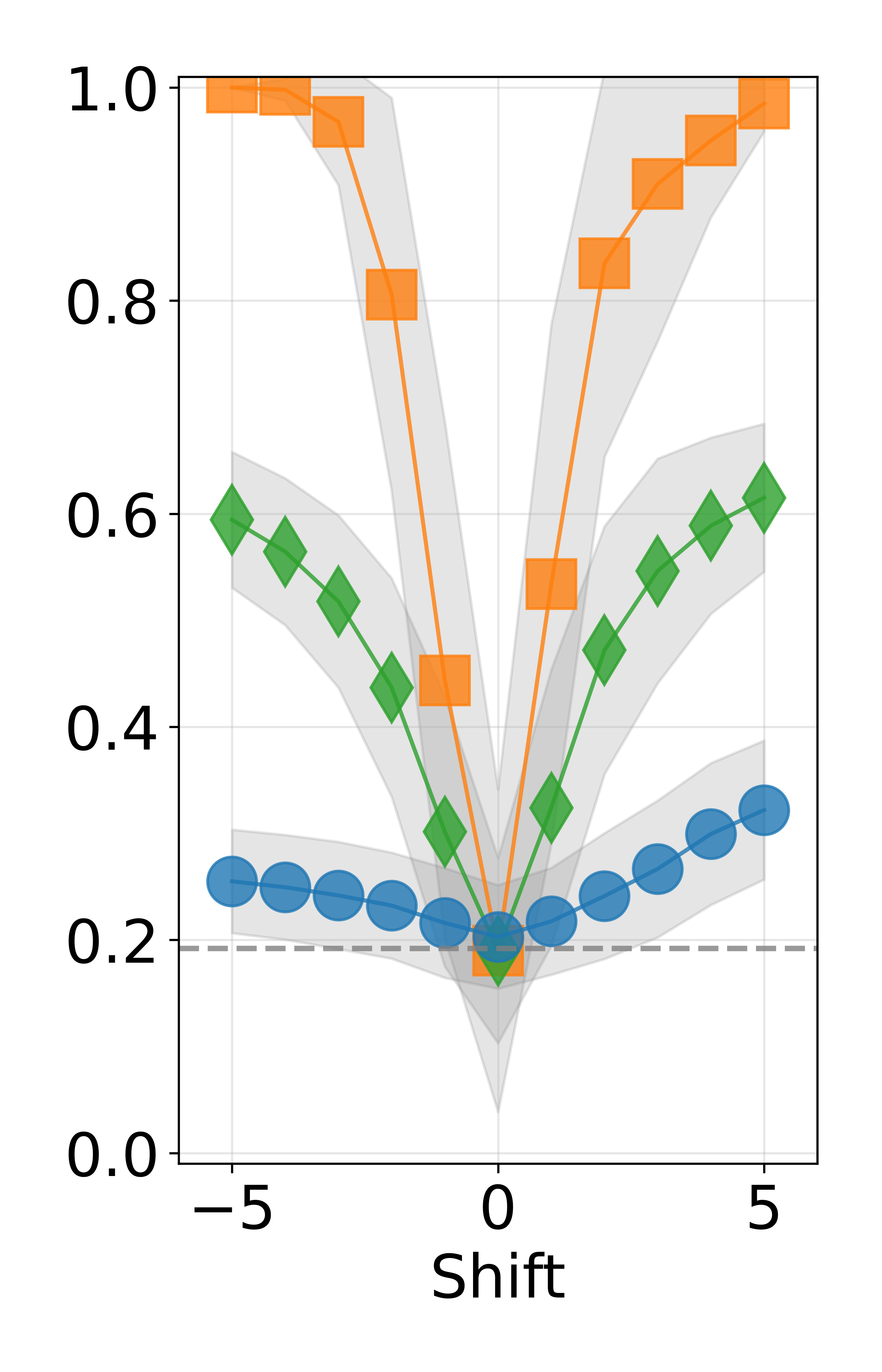}
					\caption{Scenario 3}
				\end{subfigure}		
				\caption{\textbf{Explanation faithfulness under varying mean shifts.} Panels (a)--(c) depict explanation faithfulness against the mean shift with shaded $\pm 1$ standard deviation region for Scenarios 1--3, respectively. Line styles distinguish the three methods and the random pick (LIME: circle; BREAD: square; Naive: diamond; random guess: dashed).}
				\label{fig:compar_result_lime}
			\end{figure}
			
			Figure~\ref{fig:compar_result_lime} depicts faithfulness of the three methods across three scenarios with different shifts in mean. For all methods, when there is no shift in mean, the top-$K$ feature importance is very close to the results of random pick. 
			Meanwhile, for all methods, as the magnitude of shift increases, the faithfulness score also increases. Specifically, BREAD achieves the highest faithfulness score, while LIME remains relatively low and flat. Moreover, the performance gap between these two methods increases rapidly as the shift becomes large, which is consistent with Theorem~\ref{th:faith}. In addition, the naive method improves with the size of shift as well, its faithfulness score generally falls between LIME and BREAD.
			
			In terms of computational cost, LIME and BREAD are comparable in runtime at $d=500$, whereas the naive method is four times slower. Results are detailed in Appendix~{D.5}.
			
			\subsubsection{Robustness}\label{sec:sim_robust}
			The robustness test is designed to assess how perturbation in observation affects the diagnosis results. 
			It is evaluated as the consistency of the explanation method for the same anomaly type, i.e., with the same mean shift under the same scenario. 
			For each anomaly type, we perform 1000 runs. Since all anomalous observations are generated from the same underlying shift, any variation in diagnosis reflects sensitivity to noise rather than a change in fault source.
			The robustness score for run $r$ is the average pairwise cosine similarity of the top-$K$ feature sets across all $N_r$ identified signals,
			\(
			S^{(r)}_R = \frac{1}{\binom{N_r}{2}} \sum_{i=1}^{N_r} \sum_{j=i+1}^{N_r} \frac{| D_{r,i} \cap D_{r,j}| }{\sqrt{(|D_{r,i}|\cdot|D_{r,j}|)}},
			\)
			where $D_{r,i}$ denotes the features selected in the top-$K$ for the signal $i$ in run $r$.
			The overall robustness score is the average pairwise similarity $S^{(r)}_R$ across all 1000 runs.
			Note that this metric only measures robustness, but not faithfulness.
			\begin{figure}[!h]
				\centering
				\includegraphics[width=0.35\columnwidth]{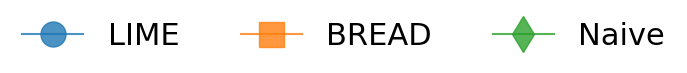}
				
				\begin{subfigure}{0.325\columnwidth}
					\centering
					\includegraphics[width=\columnwidth]{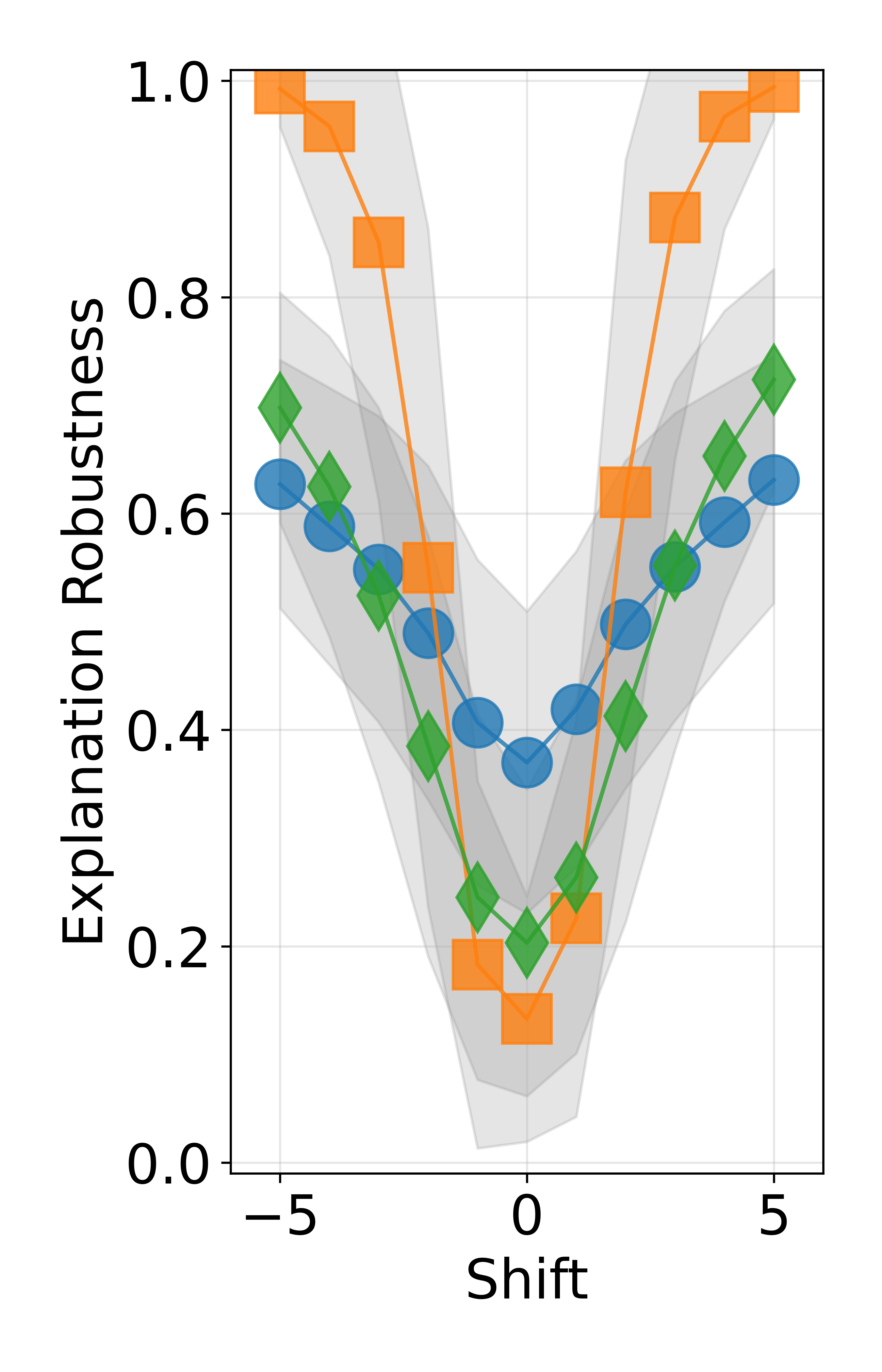}
					\caption{Scenario 1}
				\end{subfigure}
				\begin{subfigure}{0.325\columnwidth}
					\centering
					\includegraphics[width=\columnwidth]{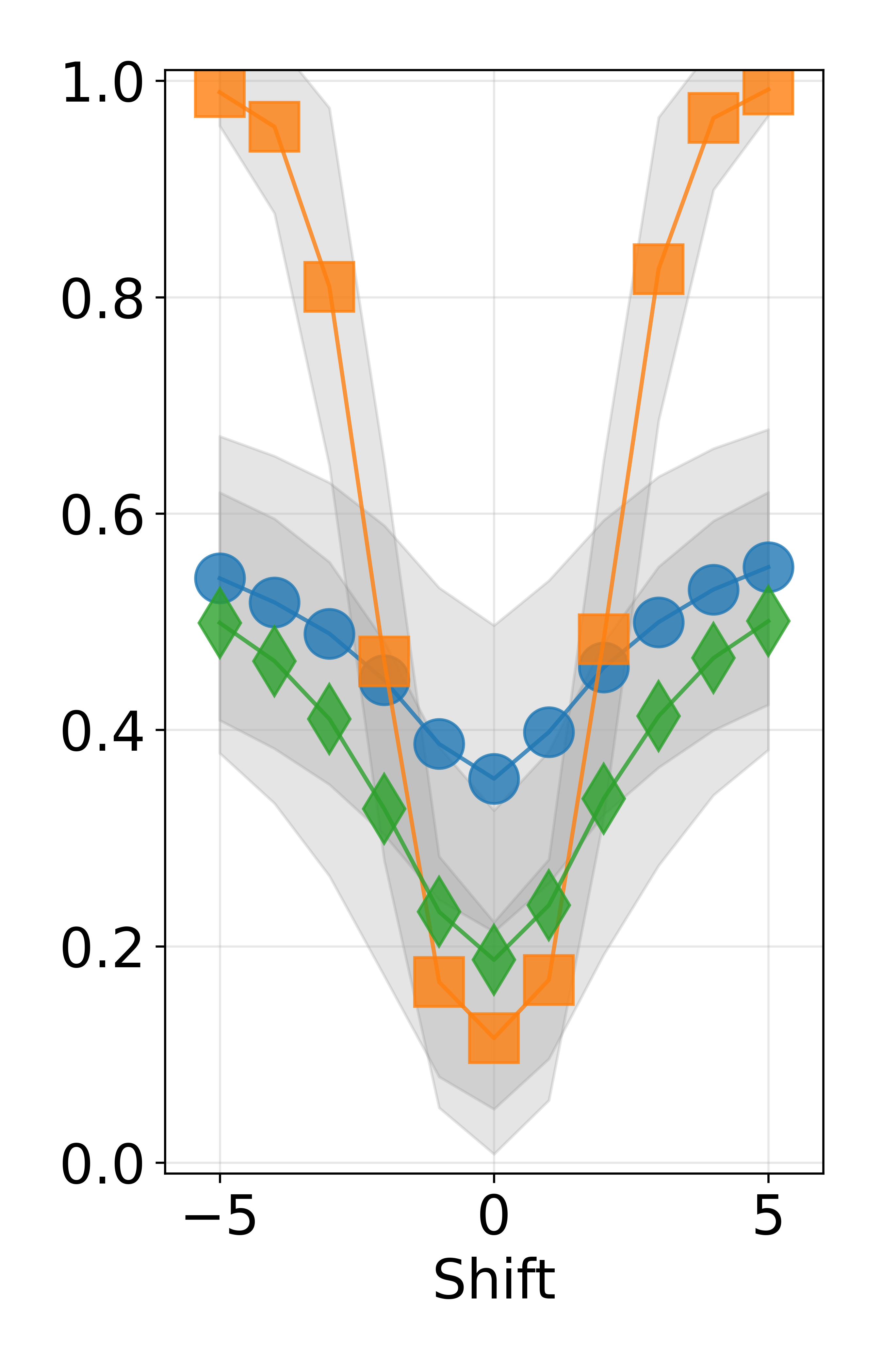}
					\caption{Scenario 2}
				\end{subfigure}
				\begin{subfigure}{0.325\columnwidth}
					\centering
					\includegraphics[width=\columnwidth]{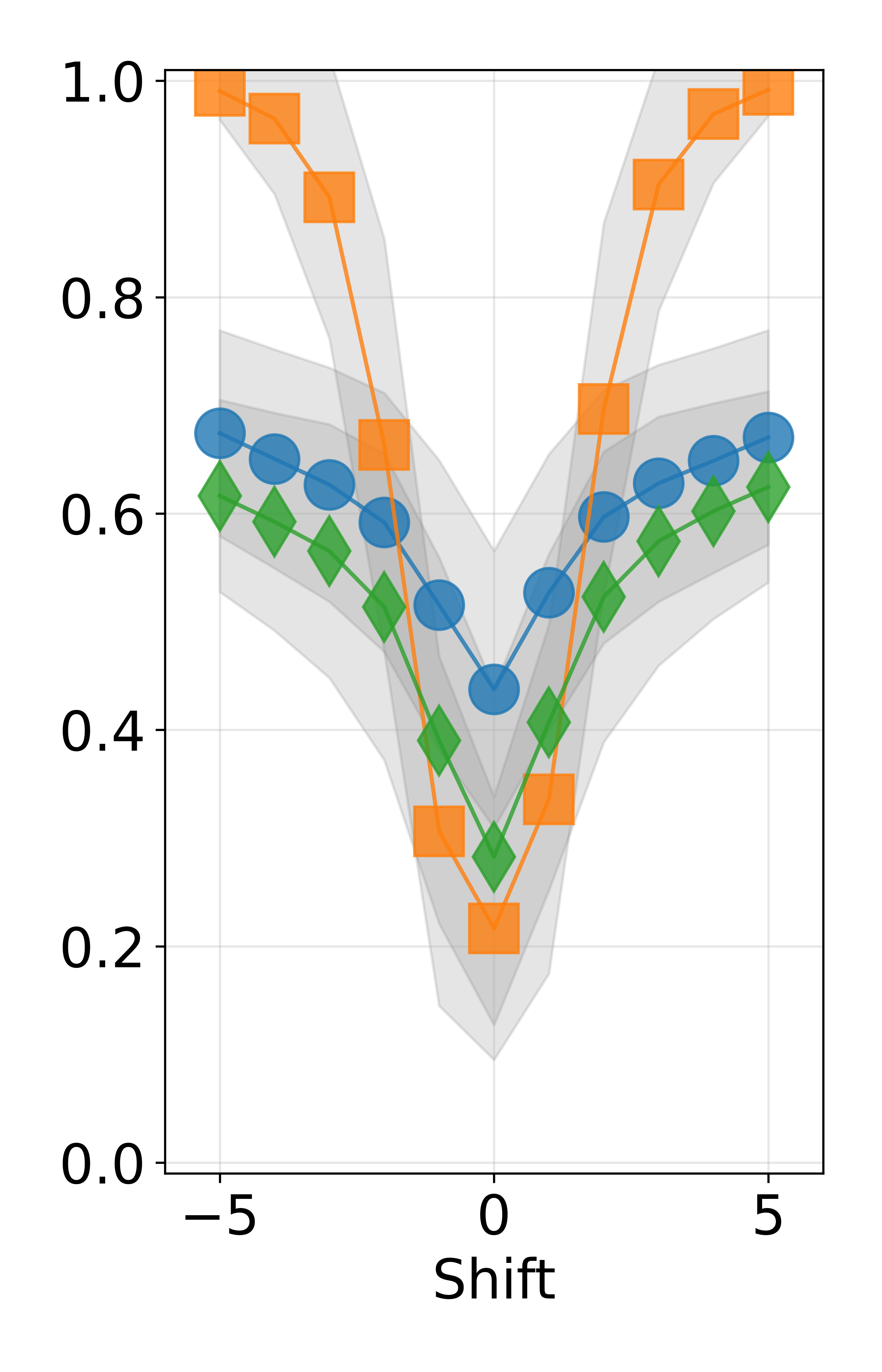}
					\caption{Scenario 3}
				\end{subfigure}
				
				\caption{\textbf{Explanation robustness under varying mean shifts.} 
					Panels (a)–(c) correspond to Scenarios 1–3, respectively. In each panel, explanation robustness is plotted with variability bands against the magnitude of the mean shift for LIME (circle), BREAD (square), and Naive (diamond).}
				\label{fig:robust_result_lime}
			\end{figure}
			
			The robustness results for the three methods are depicted in Figure~\ref{fig:robust_result_lime}.
			Robustness improves with the mean shift size for all scenarios, where consistency is lowest near the zero mean shift. This means all methods are likely to report random noise. The key result is the rapid increase in robustness as the mean shift grows, indicating consistent diagnoses once the anomaly signal is larger than the noise level.
			In addition, all methods show higher robustness under Scenario 3. 
			Across scenarios, BREAD achieves the highest robustness for moderate-to-large mean shifts. LIME shows higher robustness despite low faithfulness near zero shifts, suggesting repeated potentially incorrect feature selection when the anomaly signal is weak.
			\subsubsection{Stability}
			Stability is evaluated by repeatedly explaining the same outlier instance across 1000 random seeds at each sample size, and measuring the variability of the resulting feature attributions.
			Specifically, the set cosine similarity $S_{St}$ between the top-$K$ features and the ground truth across repeated runs is used as the metric.
			The outlier is fixed from one run under Scenario 1 with mean shift of $2$.
			This criterion is particularly relevant for sampling-based approaches such as LIME and its variants, which are sensitive to perturbation and sampling randomness.	
			\begin{figure}[!h]
				\begin{subfigure}{0.45\columnwidth}
					\centering
					\includegraphics[width=\columnwidth]{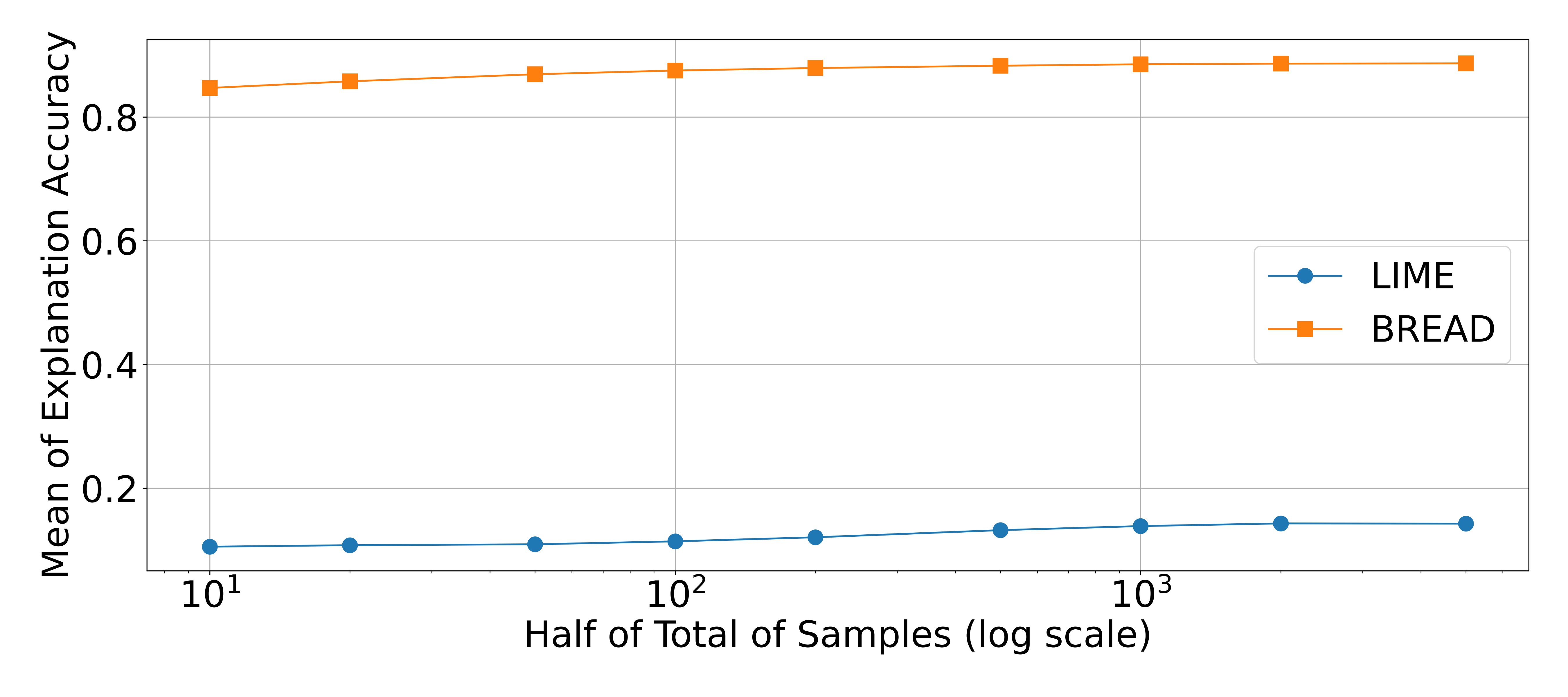}
					\caption{Mean}
					\label{fig:stability_result_lime}
				\end{subfigure}
				\hfill
				\begin{subfigure}{0.45\columnwidth}
					\centering
					\includegraphics[width=\columnwidth]{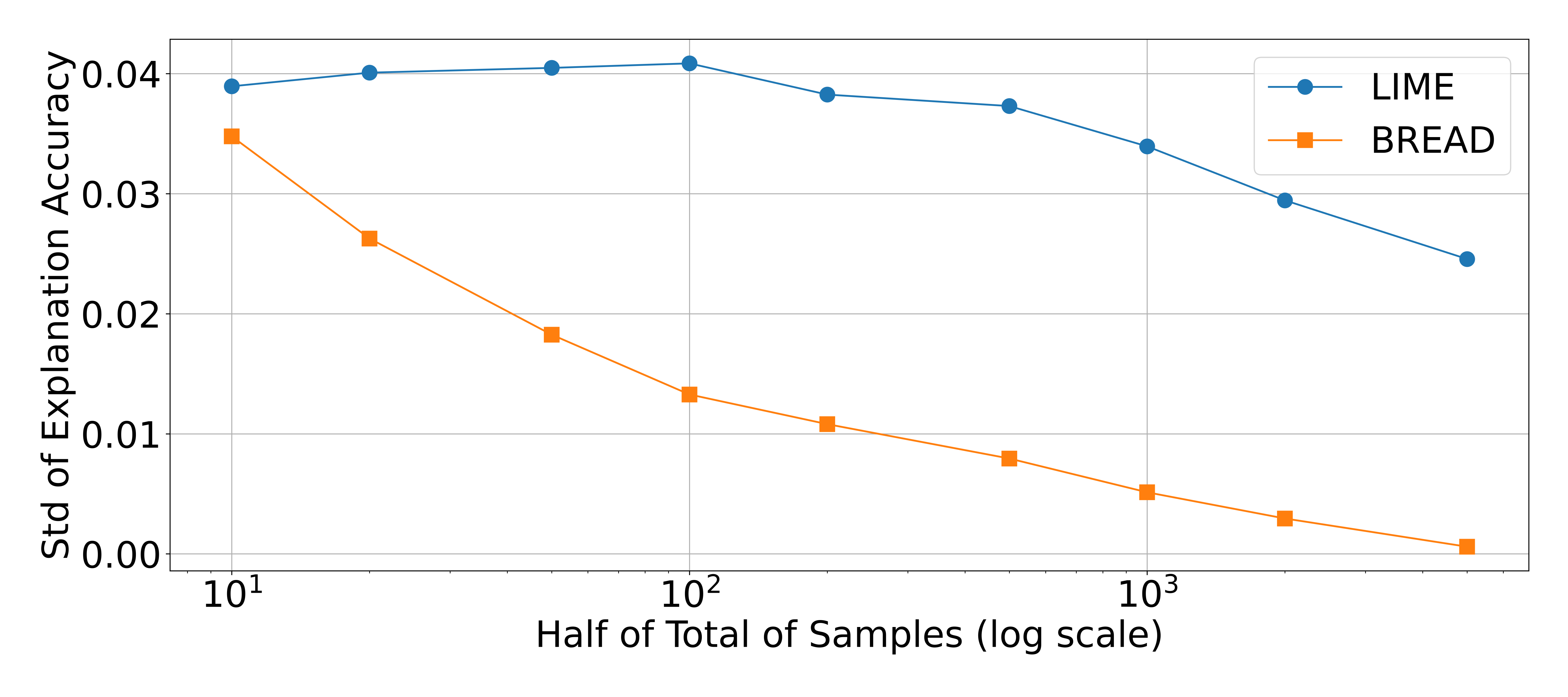}
					\caption{Standard deviation}
					\label{fig:stability_result_bread}
				\end{subfigure}
				\caption{\textbf{Stability diagnosis results.} 
					The horizontal axis shows the number of samples on a logarithmic scale.
					(a) Mean faithfulness score, repeated 1000 runs with different random seeds at each sample size. (b) Standard deviation of faithfulness score across the same repetitions.
					Curves compare LIME and BREAD, highlighting how resampling size affects both the average faithfulness score and the stability.}
				\label{fig:stability_result}
			\end{figure}
			
			Figure~\ref{fig:stability_result} shows how the sample size affects average faithfulness score and stability across repetitions. BREAD achieves a consistently high faithfulness score even at a relatively small sample size, and its standard deviation decreases rapidly as the number of samples increases. In contrast, LIME shows lower mean faithfulness score, which improves gradually with the size of sampled data, while the deviation stays elevated and declines notably only at larger sample sizes.
			
			\subsection{Case study}
			In this section, we revisit the case mentioned in the introduction and implement the proposed method in this real data set to show its applicability and performance in practice.
			\subsubsection{Data Description and Monitoring}
			We select 16 escalators and monitor their daily operating status. As a concrete fault type with known ground truth, we use the unexpected shutdowns as the fault type, which happened to Escalators 08, 13, and 14 (see Appendix~{D.8}, Figure~3 of the Supplementary).
			
			The VAE-LSTM-based $T^2$ control charts proposed by \citet{maged2024variational} are applied as the monitoring method. 
			The monitoring results are illustrated in a $T^2$ control chart in Figure~\ref{fig:framework}(c).
			Note that, since the model may raise a false alarm, we set the empirical false alarm probability (FAP) relatively low.
			
			\subsubsection{Signal Diagnosis}
			We apply both LIME and BREAD to all OOC signals detected during the shutdown period. For each signal, both methods generate a feature importance score for each of the 16 escalators, using the same number of synthetic samples ($n = 2000$). The in-control reference point used by the proposed method is estimated from the training data as the component-wise mean of the in-control observations.
			
			\begin{figure}[!h]
				\centering
				\begin{subfigure}{0.45\linewidth}
					\centering
					\includegraphics[width=\columnwidth]{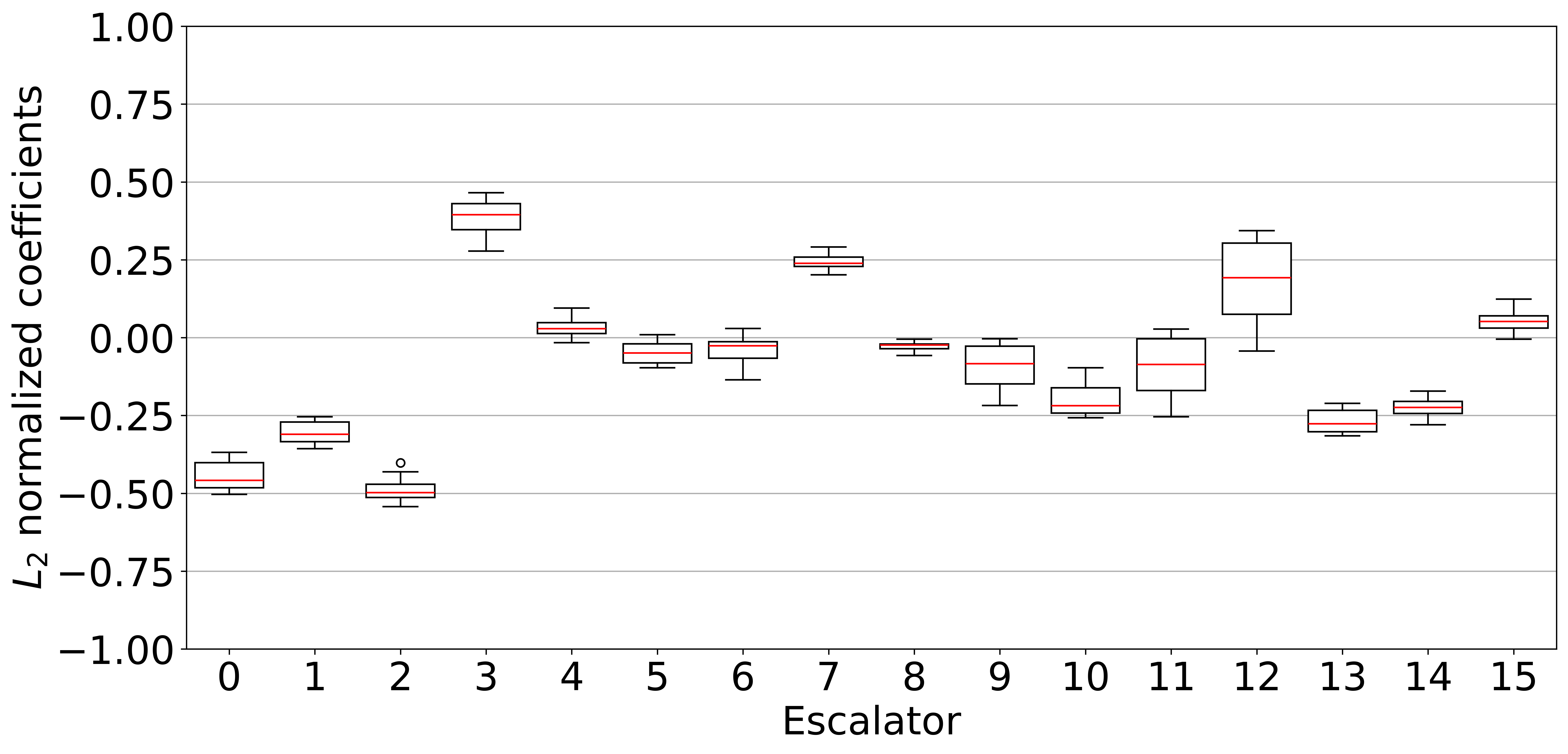}
					\caption{}
					\label{fig:caseLIME_box}
				\end{subfigure}
				\hfill
				\begin{subfigure}{0.45\linewidth}
					\centering
					\includegraphics[width=\columnwidth]{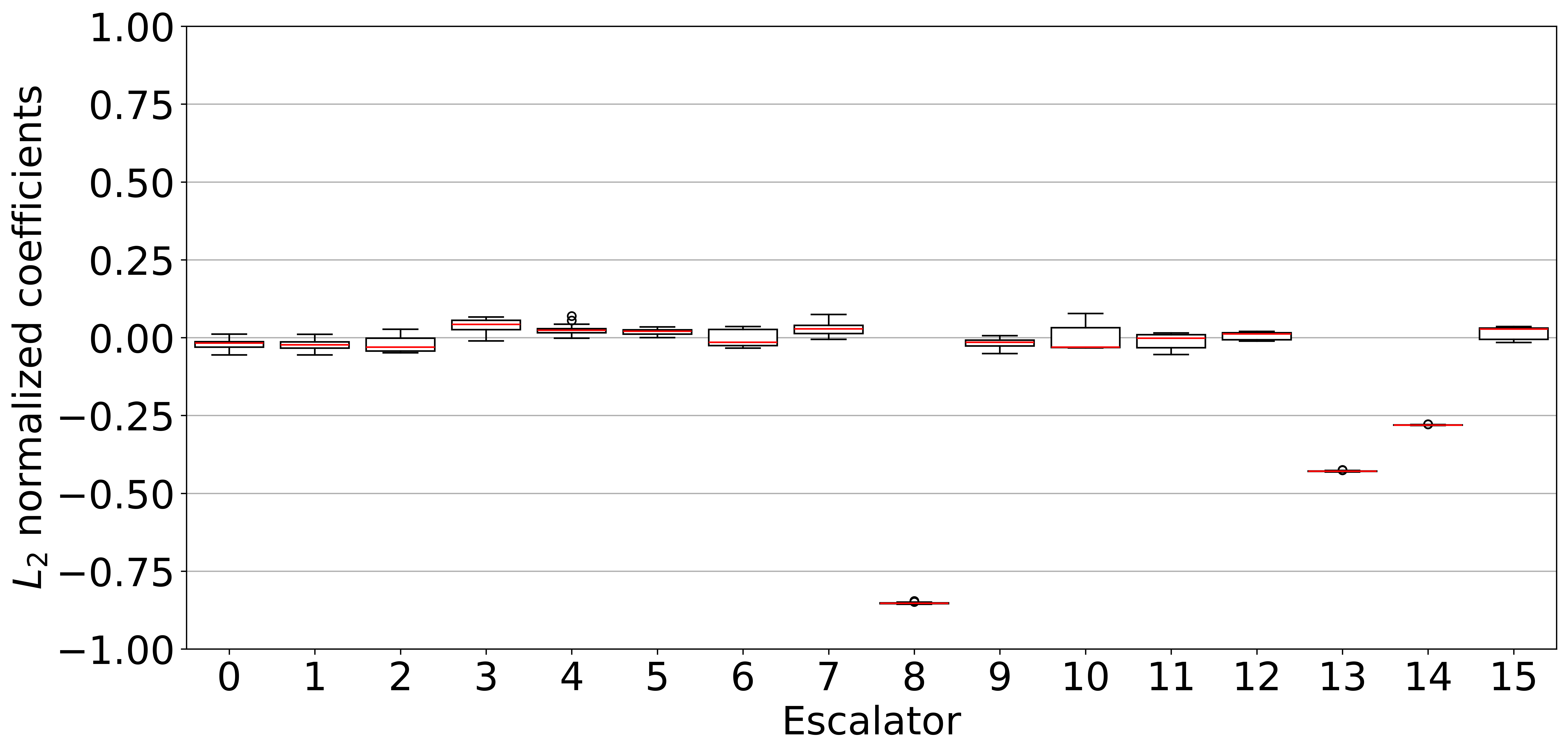}
					\caption{}
					\label{fig:caseBREAD_box}
				\end{subfigure}
				\caption{
					\textbf{Box plots of $l_2$-regularized feature importance.} Each box plot summarizes the median, range, and outliers for feature importance. 					
					(a)~LIME distributes importance broadly across non-shut-down escalators (e.g., 00--02, 03, 07, and 12), failing to concentrate relevance on the ground-truth Escalators 08, 13, 14.
					(b)~BREAD assigns the largest importance consistently to Escalators~08, 13, and~14, with near-zero, low-variance relevance for all remaining escalators, correctly recovering the fault source.
				}	 		 	
				\label{fig:caseresults_box}	
			\end{figure}
			
			The diagnosis results from all the signals are illustrated in box plots, see Figure~\ref{fig:caseresults_box}. The results show a clear qualitative difference between the two methods.

			LIME assigns large importance values to Escalators~03 and~07 (median positive importance $\approx 0.25$--$0.40$) and large negative importance to Escalators~00, 01, and~02 (median $\approx -0.25$--$-0.40$), while most remaining escalators receive near-zero scores. Critically, the three ground-truth shut down escalators (08, 13, 14) are not consistently identified among the top-ranked features, indicating that LIME's diagnosis is misaligned with the actual fault source.
			
			The BREAD method produces markedly different and more faithful results. Escalator~08 receives the largest negative importance (median $\approx -0.80$), and Escalators~13 and~14 receive consistently large negative importance (median $\approx -0.25$--$-0.50$). The remaining 13 escalators are assigned near-zero importance with low variance, reflecting the asymptotic sparsity property established in Section~\ref{sec:anaprop}. This concentration of relevance on the true fault variables is consistent across signals, demonstrating both faithfulness and stability.
			
			The proposed method correctly identifies the three shutdown escalators across all detected signals, whereas LIME consistently mis-attributes importance to normal escalators. This aligns with the theoretical sparsity result, that is, BREAD concentrates relevance on variables that drive the process away from normal operation.
			Further, the low variance of importance scores for non-faulty escalators (Figure~\ref{fig:caseresults_box}(b)) indicates that BREAD is robust to input perturbations. This is consistent with the simulation findings in Section~\ref{sec:sim_robust}.
			
			Note that, diagnosis quality is inherently coupled to the monitoring model quality. In this case study, the VAE-LSTM model is trained on a sufficiently long in-control period and validated to achieve a low FAP before deployment. In general, the proposed diagnosis method should be applied only after the monitoring model has been properly validated.
			
			\section{Conclusion and Limitation}
			In this paper, we propose the BREAD method for post-hoc diagnosis in AI-based prospective anomaly detection. By incorporating in-control reference information into local surrogate construction, BREAD generates explanations that are more aligned with the variables driving out-of-control signals while preserving the model-agnosticism needed for applicability and computational efficiency needed for online deployment.
			Theoretical analysis and empirical results on simulations and a real monitoring case study demonstrate that this global baseline-aware explanation is effective for practical diagnosis in black-box SPM control chart frameworks.
			
			Several important directions for future work remain to be explored.
			First, the current method provides point estimates of feature relevance but does not quantify explanation uncertainty. Future work can estimate the uncertainty of signal explanation for sampled local surrogates. 
			Second, the proposed model uses a single in-control reference. It would be valuable to investigate principled strategies for selection and extend to multi-reference variants when the baseline operation is multi-modal. 
			Third, although the method guarantees sparse relevance patterns, the subgroup of features remains unselected. Future work could therefore investigate feature-selection mechanisms on top of the proposed method to improve actionability in practice. 
			
			\bibliography{references}
			\bibliographystyle{plainnat}
	
	\appendix
	\onecolumn
	
	\begin{center}
		{\LARGE\bfseries BREAD: Baseline-Referenced Explanations for Anomaly Diagnosis \\ Supplementary Material\par}
		\vspace{0.5em}
		{\large Jiaqi Qiu, Rob Goedhart, Jannis Kurtz and Inez M. Zwetsloot\par}
	\end{center}
	
	\vspace{0.5em}
	
	\section{Pseudocode}\label{sec:pseudocode}
	The overall flow of our algorithm is outlined in Algorithm~\ref{alg:bread}.
	\begin{algorithm}[!h]
		\caption{BREAD}
		\label{alg:bread}
		\begin{algorithmic}
			\STATE {\bfseries Input:} AI-based black-box monitoring model $f$, identified OOC signal $\bm{t}_{j^*}$, pre-defined baseline reference point $\bm{t}_{\mathrm{ref}}$, number of samples $N$, kernel $\pi_{i}$.
			
			\STATE {\bfseries Output:} A set of the coefficients from a linear model $g$ for the anomaly score $f(\bm{t}_{j^*})$.
			\STATE
			
			\STATE $\mathcal{Z} \leftarrow [\,]$;
			\STATE $\mathbf{y} \leftarrow [\,]$;
			\STATE $\mathbf{\Pi} \leftarrow [\,]$; 
			\FOR{$i = 1$ {\bfseries to} $N/2$}
			\STATE $\bm{z}_i \leftarrow \textsc{SampleAround}\bm{t}_{j^*}$;
			\STATE $y_i \leftarrow f(\bm{z}_i)$;
			\STATE $\pi_{i} \leftarrow \pi_{i}(\bm{z}_i)$;
			\STATE $\mathcal{Z} \leftarrow \mathcal{Z} \cup \{\bm{z}_i\}$;
			\STATE $\mathbf{y} \leftarrow \mathbf{y} \cup \{y_i\}$;
			\STATE $\mathbf{\Pi} \leftarrow \mathbf{\Pi} \cup \{ \pi_{i} \}$;
			\ENDFOR
			
			\FOR{$i = N/2 +1$ {\bfseries to} $N$}
			\STATE $\bm{z}_i \leftarrow \textsc{SampleAround} \bm{t}_{\mathrm{ref}}$;
			\STATE $y_i \leftarrow f(\bm{z}_i)$;
			\STATE $\pi_{i} \leftarrow \pi_{i}(\bm{z}_i)$;
			\STATE $\mathcal{Z} \leftarrow \mathcal{Z} \cup \{\bm{z}_i\}$;
			\STATE $\mathbf{y} \leftarrow \mathbf{y} \cup \{y_i\}$;
			\STATE $\mathbf{\Pi} \leftarrow \mathbf{\Pi} \cup \{ \pi_{i} \}$;
			\ENDFOR			
			\STATE $g^* \leftarrow \displaystyle \arg\min \, \sum_{z \in \mathcal{Z}} \frac{1}{2} \pi_{i} ( g(z) - \bm{y})^{2} + \Omega(g)$,
			\STATE where $\Omega(g)$ is the ridge term.
			\STATE  {\bfseries Return}  $\textsc{ExtractFeatureImportance}(g^*)$.
		\end{algorithmic}
	\end{algorithm}

	\section{Analytical Results of the Proposed Model} \label{sec:anares}
	Denote the objective function $\frac{1}{2} ( Z {\bm{\beta}} +{\beta_0} \bm{1} - \bm{y})^{\intercal} \Pi ( Z {\bm{\beta}} + {\beta_0} \bm{1} - \bm{y}) + \frac{\lambda}{2}  {\bm{\beta}}^{ \intercal} \bm{I}_{d} {\bm{\beta}}$ as $J$.
	
	Differentiate the objective function $\mathcal{J} (\bm{ \beta }^{ \prime })$ w.r.t $\bm{ \beta }^{\prime}$, 
	we obtain
	\[ 
	\nabla_{\hat{\bm{\beta}}^{\prime}} \mathcal{J}(\hat{\bm{\beta}}^{\prime})  =  Z^{\prime \intercal} \Pi ( Z^{\prime} \hat{\bm{\beta}}^{\prime} - \bm{y}) +  \lambda \bm{I}_{d+1}  \hat{\bm{\beta}}^{\prime}.
	\]
	Set \(\nabla_{\hat{\bm{\beta}}^{\prime}} \mathcal{J}(\hat{\bm{\beta}}^{\prime}) = 0\), we get
	\begin{equation}\label{eq:FOC}
		(Z^{\prime \intercal} \Pi Z^{\prime} + \lambda \bm{I}_{d+1} )  \hat{\bm{\beta}}^{\prime} = Z^{\prime \intercal} \Pi \bm{y}.
	\end{equation}
	i.e.,
	\begin{equation}
		\left(\begin{matrix}
			Z^{ \intercal} \Pi Z  + \lambda \bm{I}_{d} & Z^{ \intercal} \Pi \bm{1} \\
			\bm{1}^{ \intercal} \Pi Z & \bm{1}^{ \intercal} \Pi \bm{1}
		\end{matrix} \right)
		\left(\begin{matrix}
			\hat{\bm{\beta}}\\
			\hat{\beta}_{0} 
		\end{matrix} \right)
		=
		\left(\begin{matrix}
			Z^{\intercal} 		\Pi  \bm{y}\\
			\bm{1}^{\intercal}	\Pi \bm{y} 
		\end{matrix} \right)
	\end{equation}
	
	Eliminating the intercept term $\hat{\beta}_{0}$ from this linear system, 
	\begin{equation}\label{eq:fullsys}
		(Z^{ \intercal} \Pi  Z + \lambda \bm{I}_{d}  - \frac{Z^{ \intercal} \Pi \bm{1} {Z^{ \intercal} \Pi \bm{1}}^{\intercal}}{\sum_{i}^{n} \pi_i})\hat{\bm{\beta}} = Z^{\intercal} \Pi \bm{y} -\frac{1}{\sum_{i}^{n} \pi_i}{Z^{ \intercal} \Pi \bm{1}}\bm{1}^{\intercal} \Pi \bm{y}.
	\end{equation}
	
	Thus, we get 
	\begin{equation}
		\hat{\bm{\beta}} = (A + \lambda \bm{I}_{d} )^{-1} \bm{b},
	\end{equation}
	where
	$A =  \left(Z^{ \intercal} \Pi  Z - \frac{Z^{ \intercal} \Pi \bm{1} \left(Z^{ \intercal} \Pi \bm{1}\right)^{\intercal}}{\sum_{i=1}^{n} \pi_i}\right),$ and 
	$\bm{b} = Z^{\intercal} \Pi \bm{y} -\frac{1}{\sum_{i=1}^{n} \pi_i}\left(Z^{ \intercal} \Pi \bm{1}\right)\left(\bm{1}^{\intercal} \Pi \bm{y}\right)$.

	\section{Faithfulness Proof} \label{sec:proof}

	\subsection{Restatement of Assumptions and Theorem}
	For convenience, we restate the assumptions and theorem.
	
	\begin{assumption}[Sampled data geometric assumption]
		There exist constants $c,\varepsilon,\rho>0$ (independent of $\delta$) and core index sets $J^{(\delta)}_0\subset I^{(\delta)}_0$ and $J^{(\delta)}_1\subset I^{(\delta)}_1$ with $|J^{(\delta)}_0|=|J^{(\delta)}_1|\ge 1$, such that 
		$ \| \bm{z}_i^{(\delta)}-\bm{t}_{\mathrm{ref}} \|_2 \le \frac{c}{\delta^2} $, for $i\in J^{(\delta)}_0,$ and $ \| \bm{z}_i^{(\delta)}-\bm{t}_{j^*} \|_2 \le \frac{c}{\delta^2} $, for $i\in J^{(\delta)}_1.$
		Furthermore, for the non-core points, 
		$ \varepsilon \le \|\bm{z}_i^{(\delta)}-\bm{t}_{\mathrm{ref}} \|_2 \le \rho$, for $i\in I^{(\delta)}_0\setminus J^{(\delta)}_0,$ and $\varepsilon \le \|\bm{z}_i^{(\delta)}-\bm{t}_{j^*} \|_2 \le \rho$, for $i\in I^{(\delta)}_1\setminus J^{(\delta)}_1.$
	\end{assumption}

	\begin{assumption}[Output structure]
		Assume there exist constants $B_y<\infty$ and $c_\Delta>0$ independent of $\delta$ such that for each $s \in\{0,1\}$, $\bigl| y_i^{(\delta)}-\bar y_s^{(\delta)} \bigr| \le B_y,$ and $\Delta(\delta) \geq c_\Delta$.   
		Moreover, for $s \in\{0,1\}$ and $i \in J_s^{(\delta)}$,
		$|y_i^{(\delta)}-\bar y_s^{(\delta)}| = O(\delta^{-1}).$
	\end{assumption}
	\begin{theorem}[Theorem~{3.3}, restated]
		Suppose the shift setup in (5) holds, the synthetic data matrix $Z^{(\delta)}$ satisfies Assumption~\ref{ass:samdata}, and the corresponding outputs, ${y}_i^{(\delta)} = f(\bm{z}_i^{(\delta)})$, satisfy Assumption~\ref{ass:out}.
		Let $\hat{\bm{\beta}}$ denote the estimator in (4) and define $\epsilon  = \dfrac{\| \hat{\bm{\beta}}_{\mathcal{R}} \|_{2} }{\| \hat{\bm{ \beta }}_{ \mathcal{K} } \|_{2}}$.
		Then,
		\[{\epsilon} \rightarrow 0, \qquad \text{as} \, \delta \rightarrow \infty.\]
	\end{theorem}
	
	\subsection{Notation and Setup}
	
	Throughout the proof, all asymptotic notation is with respect to $\delta \to \infty$ with $n$ fixed. The notation \(O(\cdot)\), \(o(\cdot)\), and \(\Omega(\cdot)\) has its standard meaning with respect to this limit. That is, for \(f(\delta)\) and \(g(\delta)>0\)
	\begin{itemize}
		\item $f = O(g)$: there exists $C > 0$ and $\delta_0 > 0$ such that $|f(\delta)| \le Cg(\delta)$ for all $\delta \ge \delta_0$;
		\item $f = o(g)$: $f(\delta)/g(\delta) \to 0$ as $\delta \to \infty$;
		\item $f = \Omega(g)$: there exists $c > 0$ and $\delta_0 > 0$ such that $f(\delta) \ge cg(\delta)$ for all $\delta \ge \delta_0$.
	\end{itemize}

	We denote the synthetic sample set as $\mathcal{Z}^{(\delta)}=\{\bm{z}_i^{(\delta)}\}_{i=1}^n$ and the corresponding matrix $Z$,  and index sets $I_0^{(\delta)}$ and $I_1^{(\delta)}$ with $|I_0^{(\delta)}|=|I_1^{(\delta)}|=n/2$ containing samples around $\bm{t}_{\mathrm{ref}}$ and $\bm{t}_{j^*}$, respectively. Under Assumption \ref{ass:samdata}, we denote core index sets $J^{(\delta)}_0\subset I^{(\delta)}_0$ and $J^{(\delta)}_1\subset I^{(\delta)}_1$ with $|J^{(\delta)}_0|=|J^{(\delta)}_1|\geq 1$. For the non-core points, we define the non-core index set as
	\[
	T^{(\delta)} 
	= \bigl( I_0^{(\delta)} \setminus J_0^{(\delta)} \bigr) \cup \bigl(I_1^{(\delta)} \setminus J_1^{(\delta)} \bigr). 
	\] 
	We consider a setting that a signal is triggered by a shift, denoted by
	$\bm{t}_{j^*} = \bm{t}_{\mathrm{ref}} + \delta \bm{u}$,
	where $\bm{u}$ is a unit vector and $\delta := \|\bm{t}_{j^*} - \bm{t}_{\mathrm{ref}}\|_2$.
	Let $\mathcal{K} = \{\ell : {u}_\ell \neq 0\}$ denote the set of shifted features with $|\mathcal{K}| = k$, and $\mathcal{R}=\{1,\dots,d\}\setminus\mathcal{K}$, we split the feature space
	\[
	\mathbb{R}^d = \mathbb{R}^\mathcal{K} \oplus \mathbb{R}^\mathcal{R},
	\]
	where $\mathbb{R}^\mathcal{K}$ is the subspace of features leading $\bm{t}_{j^*}$ to a signal and $\mathbb{R}^\mathcal{R}$ is its unshifted complement of $\bm{t}_{j^*}$.
	Thus, for a vector $\bm{b}\in\mathbb{R}^d$, we write $\bm{b}_{\mathcal{K}}$ and $\bm{b}_{\mathcal{R}}$ for its sub-vectors indexed by $\mathcal{K}$ and $\mathcal{R}$, respectively.
	
	Similarly, for a matrix $M\in\mathbb{R}^{d \times d}$, we write $M_{\mathcal{KK}}, M_{\mathcal{KR}}, M_{\mathcal{RK}}$, and $M_{\mathcal{RR}}$ for the corresponding block sub-matrices.
	Adapting the split, we write the linear system \eqref{eq:fullsys} as
	\begin{equation*}
		\begin{aligned}
			M \hat{ \bm{ \beta}} & = \bm{b},\\
			\begin{bmatrix}
				M_{\mathcal{KK}} & M_{\mathcal{KR}}\\
				M_{\mathcal{RK}} & M_{\mathcal{RR}}
			\end{bmatrix}
			\begin{bmatrix}
				\hat{ \bm{ \beta}}_{\mathcal{K}} \\
				\hat{ \bm{ \beta}}_{\mathcal{R}}
			\end{bmatrix}
			& = 
			\begin{bmatrix}
				\bm{b}_{\mathcal{K}}\\
				\bm{b}_{\mathcal{R}}
			\end{bmatrix},
		\end{aligned}
	\end{equation*}
	where $M := A + \lambda \bm{I} = (Z^{ \intercal} \Pi  Z - \frac{Z^{ \intercal} \Pi \bm{1} {Z^{ \intercal} \Pi \bm{1}}^{\intercal}}{\sum_{i}^{n} \pi_i}) +\lambda \bm{I} $, $M_{\mathcal{KK}} \in \mathbb{R}^{k \times k}$, $M_{\mathcal{RR}} \in \mathbb{R}^{(d-k) \times (d-k)}$,  $M_{\mathcal{KR}} \in \mathbb{R}^{ k \times (d-k) }$, and $\bm{b} = Z^{\intercal} \Pi \bm{y} -\frac{1}{\sum_{i}^{n} \pi_i}{Z^{ \intercal} \Pi \bm{1}}\bm{1}^{\intercal} \Pi \bm{y}$.
	Therefore, we obtain 
	\[
	\epsilon = \frac{\| \hat{\bm{\beta}}_{\mathcal{R}} \|_2}{\| \hat{\bm{\beta}}_{\mathcal{K}} \|_2}, 
	\; \text{and} \;
	\begin{aligned}
		\begin{bmatrix}
			\hat{ \bm{ \beta}}_{\mathcal{K}} \\
			\hat{ \bm{ \beta}}_{\mathcal{R}}
		\end{bmatrix}
		& =
		\begin{bmatrix}
			M_{\mathcal{KK}} & M_{\mathcal{KR}}\\
			M_{\mathcal{RK}} & M_{\mathcal{RR}}
		\end{bmatrix}^{-1}
		\begin{bmatrix}
			\bm{b}_{\mathcal{K}}\\
			\bm{b}_{\mathcal{R}}
		\end{bmatrix}
	\end{aligned}
	\]
	
	Let \(s_\pi := \sum_{i=1}^{n} \pi_{i}\) denote the sum of the weights. We write
	\[
	s_0 := \sum_{i\in I_0^{(\delta)}}\pi_i, \qquad 
	s_1 := \sum_{i\in I_1^{(\delta)}}\pi_i.
	\]
	
	The weighted feature mean \(\bm{\mu}\) and weighted output mean \( \mu_y \) are defined as
	\[
	\bm{\mu} = \frac{(Z^{(\delta)})^\intercal \Pi \bm{1}}{s_{\pi}},
	\qquad
	\mu_{y} = \frac{\sum_{i=1}^n\pi_{i} y_{i}^{(\delta)}}{s_{\pi}}.
	\]
	
	We can write
	\[\bm{b} = {(Z^{(\delta)})}^{\intercal} \Pi \bm{y} -\frac{1}{\sum_{i}^{n} \pi_i}{{(Z^{(\delta)})}^{ \intercal} \Pi \bm{1}}\bm{1}^{\intercal} \Pi \bm{y} 
	= \sum_{i=1}^{n} \pi_i (\bm{z}_i^{(\delta)} - \bm{\mu} )({y}^{(\delta)}_i - \mu_y) .\]
	
	For the norm in the proof, $\|\cdot\|_2$ denotes two different objects depending on context.
	For a vector $\bm{x} \in \mathbb{R}^p$, $\|\bm{x}\|_2 := (\sum_{j=1}^p x_j^2)^{1/2}$ denotes the Euclidean norm.
	For a matrix $A \in \mathbb{R}^{p \times q}$, $\|A\|_2 := \sup_{\|\bm{x}\|_2=1}\|A\bm{x}\|_2$ denotes the spectral norm (largest singular value).
	For a symmetric positive semidefinite matrix $A$, $\|A\|_2$ equals its largest eigenvalue.
	We write $A \succeq B$ to mean $A - B$ is positive semidefinite.

	\subsection{Bounding $\epsilon$ by $\delta$}\label{sec:appx_eB}
	\subsubsection{Schur-Complement Reduction}
	To prove Theorem~\ref{th:faith}, we first need to get the expressions of \(\hat{ \bm{ \beta}}_{\mathcal{K}}\) and \(\hat{ \bm{ \beta}}_{\mathcal{R}}\) by the Schur complement.
	Specifically, given a symmetric SPD matrix $M$, 
	\[
	M = \begin{bmatrix}
		P 			  & Q \\
		Q^{\intercal} & R
	\end{bmatrix},
	\]
	we can write the inversion as
	\[
	M^{-1} = \begin{bmatrix}
		(P - QR^{-1}Q^{\intercal})^{-1} & -(P - QR^{-1}Q^{\intercal})^{-1}QR^{-1} \\
		-R^{-1}Q^{\intercal}(P - QR^{-1}Q^{\intercal})^{-1} & 
		R^{-1} + R^{-1}Q^{\intercal}(P - Q R^{-1}Q^{\intercal})^{-1}QR^{-1}
	\end{bmatrix}.
	\]
	Then, we can write the estimator 
	\[
	\begin{aligned}
		\begin{bmatrix}
			\hat{ \bm{ \beta}}_{\mathcal{K}} \\
			\hat{ \bm{ \beta}}_{\mathcal{R}}
		\end{bmatrix}
		& =
		\begin{bmatrix}
			M_{\mathcal{KK}} & M_{\mathcal{KR}}\\
			M_{\mathcal{RK}} & M_{\mathcal{RR}}
		\end{bmatrix}^{-1}
		\begin{bmatrix}
			\bm{b}_{\mathcal{K}}\\
			\bm{b}_{\mathcal{R}}
		\end{bmatrix}\\
		& = 
		\begin{bmatrix}
			S^{-1} & -S^{-1} M_{\mathcal{KR}} M_{\mathcal{RR}}^{-1} \\
			- M_{\mathcal{RR}}^{-1} M_{\mathcal{KR}}^{\intercal} S^{-1} & 
			M_{\mathcal{RR}}^{-1} + M_{\mathcal{RR}}^{-1} M_{\mathcal{KR}}^{\intercal} S^{-1} M_{\mathcal{KR}} M_{\mathcal{RR}}^{-1}
		\end{bmatrix}
		\begin{bmatrix}
			\bm{b}_{\mathcal{K}}\\
			\bm{b}_{\mathcal{R}}
		\end{bmatrix}
	\end{aligned}
	\]
	where $S = M_{\mathcal{KK}} - M_{\mathcal{KR}} M_{\mathcal{RR}}^{-1} M_{\mathcal{KR}}^{\intercal}$ is the Schur complement of $M_{\mathcal{RR}}$ in $M$.
	Thus, 
	\[
	\begin{aligned}
		& \hat{ \bm{ \beta}}_{\mathcal{K}} = S^{-1} \bm{b}_{\mathcal{K}} - S^{-1} M_{\mathcal{KR}} M^{-1}_{\mathcal{RR}} \bm{b}_{\mathcal{R}} =  S^{-1}\left(\bm{b}_{\mathcal{K}} - M_{\mathcal{KR}} M^{-1}_{\mathcal{RR}} \bm{b}_{\mathcal{R}}  \right) \\
		&  \hat{ \bm{ \beta}}_{\mathcal{R}} = - M_{\mathcal{RR}}^{-1} M_{\mathcal{KR}}^{\intercal} S^{-1} \bm{b}_{\mathcal{K}} +( M_{\mathcal{RR}}^{-1} + M_{\mathcal{RR}}^{-1} M_{\mathcal{KR}}^{\intercal} S^{-1} M_{\mathcal{KR}} M_{\mathcal{RR}}^{-1}) \bm{b}_{\mathcal{R}} 
		= M_{\mathcal{RR}}^{-1} \left( \bm{b}_{\mathcal{R}} - M_{\mathcal{KR}}^{\intercal} \hat{ \bm{ \beta}}_{\mathcal{K}} \right)
	\end{aligned}
	\]
	
	Thus, we can write $\epsilon$ as,
	\[
	\begin{aligned}
		\epsilon & = \frac{\| \hat{\bm{\beta}}_{\mathcal{R}} \|_2}{\| \hat{\bm{\beta}}_{\mathcal{K}} \|_2} \\
		& = \frac{\| (M_{\mathcal{RR}})^{-1} \left( \bm{b}_{\mathcal{R}} - (M_{\mathcal{KR}})^{\intercal} \hat{ \bm{ \beta}}_{\mathcal{K}} \right) \|_2}
		{\| \hat{ \bm{ \beta}}_{\mathcal{K}} \|_2} \\
		& \leq \frac{ \| (M_{\mathcal{RR}})^{-1} \|_2 (\| \bm{b}_{\mathcal{R}} \|_2 + \| M_{\mathcal{KR}}  \|_2 \| \hat{ \bm{ \beta}}_{\mathcal{K}}  \|_2 )}{\| \hat{ \bm{ \beta}}_{\mathcal{K}} \|_2}\\
		& = \| (M_{\mathcal{RR}})^{-1} \|_2 \left( \frac{\| \bm{b}_{\mathcal{R}} \|_2}{\| \hat{\bm{\beta}}_{\mathcal{K}}\|_2 }  + \| M_{\mathcal{KR}} \|_2 \right).
	\end{aligned}
	\]
	Since $M_{\mathcal{RR}} \succeq \lambda \bm{I}$, it follows that,
	\[
	\epsilon \leq
	\frac{1}{\lambda} \left( \frac{\|\bm{b}_{\mathcal{R}}\|_2}{\|\hat{\bm{\beta}}_{\mathcal K}\|_2} + \|M_{\mathcal{KR}}\|_2 \right).
	\]
	Accordingly, the proof is divided into three parts. First, we show that $\|\bm{b}_{\mathcal{R}}\|_2$ is small under the BREAD weighting. Second, we show $\|M_{\mathcal{KR}}\|_2$ is asymptotically small. Third, we show the lower bound
	on $\|\hat{\bm{\beta}}_{\mathcal K}\|_2$. These imply that $\epsilon \rightarrow 0,$ as $\delta\rightarrow\infty$.

	\subsubsection{Preliminary Bounds}
	To bound $\epsilon$ by $\delta$, we need the following lemmas that repeatedly used throughout.
	\begin{lemma}[BREAD weight bounds] \label{lem:BREAD_weights}
		Under Assumption~\ref{ass:samdata}, there exist constants \(c_B>0\), \(K_B<\infty\), and \(C_0>0\), independent of \(\delta\), such that for all sufficiently large \(\delta\):
		\begin{enumerate}
			\item for every \(i\in T^{(\delta)}:=(I_0^{(\delta)}\setminus J_0^{(\delta)})\cup (I_1^{(\delta)}\setminus J_1^{(\delta)})\),
			\[
			\pi_i\le K_B e^{-c_B\delta}.
			\]
			\item for every \(i\in J_0^{(\delta)}\cup J_1^{(\delta)}\),
			\[
			e^{-C_0/\delta}\le \pi_i \le 1;
			\]
		\end{enumerate}
		Then
		\[
		s_\pi =\sum_{i\in J_0^{(\delta)}\cup J_1^{(\delta)}}\pi_i + O(e^{-c_B\delta}) = |J_0^{(\delta)}|+ |J_1^{(\delta)}| + O(\delta^{-1}) + O(e^{-c_B\delta}),
		\]
		and $1/s_\pi=O(1)$.
	\end{lemma}
	\begin{proof}
		Recall the BREAD weights
		\[
		\pi_i = \exp\!\left(-\frac{\|\bm{z}_i^{(\delta)}-\bm{t}_{j^*}\|_2\|\bm{z}_i^{(\delta)}-\bm{t}_{\mathrm{ref}}\|_2}{\sigma^2}\right).
		\]
		
		(i) For $i \in I_1^{(\delta)}\setminus J_1^{(\delta)}$, Assumption~\ref{ass:samdata} gives
		\[\varepsilon\leq\|\bm{z}_i^{(\delta)}-\bm{t}_{j^*}\|_2\leq\rho.
		\]
		Since \(\bm{t}_{j^*} = \bm{t}_{\mathrm{ref}}+\delta \bm{u}\),
		\[
		\bm{z}_i^{(\delta)} - \bm{t}_{\mathrm{ref}} = \delta \bm{u} + (\bm{z}_i^{(\delta)} - \bm{t}_{j^*})
		\]
		Hence, by triangle inequality,
		\[\|\bm{z}_i^{(\delta)}-\bm{t}_{\mathrm{ref}}\|_2  
		=\|\delta \bm{u} + (\bm{z}_i^{(\delta)} - \bm{t}_{j^*})\|_2 
		\geq \delta -  \| \bm{z}_i^{(\delta)} - \bm{t}_{j^*}\|_2 \geq \delta-\rho.
		\]
		
		Hence, for all $\delta>\rho$,
		\[
		\pi_i\le \exp\!\left(-\frac{\varepsilon(\delta-\rho)}{\sigma^2}\right).
		\]
		
		For $i \in I_0^{(\delta)}\setminus J_0^{(\delta)}$, similarly,
		\[
		\|\bm{z}_i^{(\delta)}-\bm{t}_{\mathrm{ref}}\|_2\geq \varepsilon,
		\qquad
		\|\bm{z}_i^{(\delta)}-\bm{t}_{j^*}\|_2 \geq \delta - \rho
		\]
		
		Thus, the same bound holds. Hence there exist $K_B < \infty$ and $c_B>0$, such that
		\[
		\pi_i\le K_B e^{-c_B\delta}
		\qquad \forall i\in T^{(\delta)}.
		\]
		
		(ii) For $i\in J_0^{(\delta)}$, then $\|\bm{z}_i^{(\delta)}-\bm{t}_{\mathrm{ref}}\|_2\le c/\delta^2$, it holds
		$
		\|\bm{z}_i^{(\delta)}-\bm{t}_{j^*}\|_2=\|\delta \bm{u}-(\bm{z}_i^{(\delta)}-\bm{t}_{\mathrm{ref}})\|_2\leq \delta+\frac{c}{\delta^2}.
		$
		Therefore
		\[
		\|\bm{z}_i^{(\delta)}-\bm{t}_{\mathrm{ref}}\|_2\|\bm{z}_i^{(\delta)}-\bm{t}_{j^*}\|_2
		\leq \frac{c}{\delta}+\frac{c^2}{\delta^4},
		\]
		Then,
		\[
		\pi_i \geq \exp\left(-\frac{c/\delta + c^2/\delta^4}{\sigma^2}\right),
		\qquad i\in J_0^{(\delta)}.
		\]
		Hence,
		\[
		1 - \pi_i =O(\delta^{-1})
		\]
		The same estimate holds for $i\in J_1^{(\delta)}$. In particular, core weights are bounded away from $0$
		uniformly for large $\delta$.
		Since \(s_\pi = \sum_{i=1}^n \pi_i\),
		\(s_\pi = \sum_{i\in J_0^{(\delta)}\cup J_1^{(\delta)}}\pi_i + O(e^{-c_B\delta})\). Since \(\pi_i = 1 + O(\delta^{-1})\) and \(n\) is fixed, 
		we obtain
		\[
		s_\pi = |J_0^{(\delta)}|+ |J_1^{(\delta)}| + O(\delta^{-1}) + O(e^{-c_B\delta}).
		\]
		Since $|J_0^{(\delta)}|+|J_1^{(\delta)}|\ge 2$, we have $s_\pi \ge 2$ for large $\delta$,
		so $1/ s_\pi = O(1)$. \qedhere
	\end{proof}
	
	\begin{lemma}[Bounds of weighted sum ratio]\label{lem:weightedsum_r}
		Recall
		\[
		s_0=\sum_{i\in I_0^{(\delta)}} \pi_i, \qquad s_1=\sum_{i\in I_1^{(\delta)}} \pi_i, \qquad s_\pi=s_0+s_1.
		\]
		In particular, there exists a constant \(c_0>0\), independent of \(\delta\), such that
		\begin{equation*}
			\frac{s_0 s_1}{(s_\pi)^2}\ge c_0.
		\end{equation*}
	\end{lemma}
	\begin{proof}
		By Lemma~\ref{lem:BREAD_weights}(ii), each core set contains at least one point and each core weight is bounded below by \(e^{-C_0/\delta}\). Hence, for all sufficiently large \(\delta\),
		\[
		s_0\geq {e^{-C_0/\delta}}, \qquad s_1\geq e^{-C_0/\delta}, \qquad s_\pi \le n.
		\]
		Hence,
		\[
		\frac{s_0 s_1}{(s_\pi)^2}\geq \frac{e^{-2C_0/\delta}}{n^2} \rightarrow \frac{1}{n^2} \ge 0,
		\]
		Taking $c_0$ for all large $\delta$ gives the results. \qedhere
	\end{proof}
	
	\begin{lemma}[Pointwise bounds] \label{lem:pt_bounds}
		Under Assumption~\ref{ass:samdata}, for all sufficiently large $\delta$, 
		\begin{enumerate}
			\item[(i)] $\mathcal{R}$-block, core points.\\
			For $i\in J_0^{(\delta)}\cup J_1^{(\delta)}$,
			\[
			\|\bm{z}_{i,\mathcal{R}}^{(\delta)}-\bm{t}_{\mathrm{ref},\mathcal{R}}\|_2=O(\delta^{-2}).
			\]
			\item[(ii)] $\mathcal{R}$-block, tail points.\\
			For $i\in T^{(\delta)}$,
			\[
			\|\bm{z}_{i,\mathcal{R}}^{(\delta)}-\bm{t}_{\mathrm{ref},\mathcal{R}}\|_2=O(1).
			\]
			
		\end{enumerate}
	\end{lemma}
	\begin{proof}
		\begin{enumerate}
			\item[(i)] For $i\in J_0^{(\delta)} \cup J_1^{(\delta)},$ since $\bm{u}_{\mathcal{R}}=0$, $\bm{t}_{j^*,\mathcal{R}}=\bm{t}_{\mathrm{ref},\mathcal{R}}$. For $i\in J_0^{(\delta)}$, Assumption~\ref{ass:samdata} gives
			$\|\bm{z}_i^{(\delta)}-\bm{t}_{\mathrm{ref}}\|_2 \le c/\delta^2$. For $i\in J_1^{(\delta)}$, the same follows. 
			
			Thus,
			\[\left\|  \bm{z}_{i,\mathcal{R}}^{(\delta)} - \bm{t}_{\mathrm{ref},\mathcal{R}} \right\|_2 =O(\delta^{-2}).\]
			
			\item[(ii)] For $i\in T^{(\delta)},$ Assumption~\ref{ass:samdata} gives 
			$ \varepsilon \le \|\bm{z}_i^{(\delta)}-\bm{t}_{\mathrm{ref}} \|_2 \le \rho$, for $i\in I^{(\delta)}_0\setminus J^{(\delta)}_0,$ and $\varepsilon \le \|\bm{z}_i^{(\delta)}-\bm{t}_{j^*} \|_2 \le \rho$, for $i\in I^{(\delta)}_1\setminus J^{(\delta)}_1.$
			
			Thus, the $\mathcal{R}$-block is bounded by $\rho = O(1)$. \qedhere
		\end{enumerate}
	\end{proof}
	
	\begin{lemma}[Bounds for pointwise weighted mean]\label{lem:z_wm}
		Under Assumption~\ref{ass:samdata},
		\begin{enumerate}
			\item[(i)] $\mathcal{K}$-block, all points.\\
			For all $i\in\{1,\dots,n\}$,
			\[
			\|\bm{z}_{i,\mathcal{K}}^{(\delta)}-\bm{\mu}_{\mathcal{K}}\|_2=O(\delta).
			\]
			\item[(ii)]$\mathcal{R}$-block, all points.\\
			\[
			\bm{\mu}_{\mathcal{R}}
			=\bm{t}_{\mathrm{ref},\mathcal{R}}+O(\delta^{-2})+O(K_Be^{-c_B\delta}).
			\]
			Furthermore, for $i\in J_0^{(\delta)}\cup J_1^{(\delta)}$,
			\[
			\|\bm{z}_{i,\mathcal{R}}^{(\delta)}-\bm{\mu}_{\mathcal{R}}\|_2=O(\delta^{-2})+O(e^{-c_B\delta}),
			\]
			and for $i\in T^{(\delta)}$,
			\[
			\|\bm{z}_{i,\mathcal{R}}^{(\delta)}-\bm{\mu}_{\mathcal{R}}\|_2=O(1).
			\]
		\end{enumerate}
	\end{lemma}
	\begin{proof}
		\begin{enumerate}		
			\item[(i)] For $\bm{z}_{i,\mathcal{K}}^{(\delta)}$, Assumption~\ref{ass:samdata} gives either $\bm{z}_{i,\mathcal{K}}^{(\delta)} = \bm{t}_{\mathrm{ref},\mathcal{K}} + O(1)$ or $\bm{z}_{i,\mathcal{K}}^{(\delta)} = \bm{t}_{\mathrm{ref},\mathcal{K}} + \delta \bm{u}_{\mathcal{K}}+O(1)$.
			Thus $\bm{z}_{i,\mathcal{K}}^{(\delta)}$ lies at distance $O(\delta)$ from $\bm{t}_{\mathrm{ref},\mathcal{K}}$. Since $\bm{\mu}_{\mathcal{K}}$ is a weighted average of $\bm{z}_{i,\mathcal{K}}^{(\delta)}$'s, it also satisfies $\bm{\mu}_{\mathcal{K}} = \bm{t}_{\mathrm{ref},\mathcal{K}} + O(\delta)$.
			Hence, 
			\begin{equation*}
				\|\bm{z}_{i,\mathcal{K}}^{(\delta)} - \bm{\mu}_{\mathcal{K}}\|_2 = O(\delta). 
			\end{equation*}
			\item[(ii)]
			Write 
			\[\bm{\mu}_{\mathcal{R}} - \bm{t}_{\mathrm{ref},\mathcal{R}} = \frac{1}{s_\pi} \sum_{i}^{n} \pi_i \left( \bm{z}_{i,\mathcal{R}}^{(\delta)} - \bm{t}_{\mathrm{ref},\mathcal{R}} \right). \]
			For $i\in J_0^{(\delta)}\cup J_1^{(\delta)}$, by Lemmas~\ref{lem:BREAD_weights} and~\ref{lem:pt_bounds}(i)
			\[\left\| \bm{\mu}_{\mathcal{R}} - \bm{t}_{\mathrm{ref},\mathcal{R}}\right\|_2 = \left\| \frac{1}{s_\pi} \sum_{i}^{n} \pi_i \left( \bm{z}_{i,\mathcal{R}}^{(\delta)} - \bm{t}_{\mathrm{ref},\mathcal{R}} \right) \right\|_2 = O(\delta^{-2}) \sum_{i\in J_0^{(\delta)} \cup J_1^{(\delta)}} \pi_i = O(\delta^{-2}). \]
			Similarly, for $i\in T^{(\delta)}$, Lemmas~\ref{lem:BREAD_weights} and~\ref{lem:pt_bounds}(ii) give 
			\[
			\left\| \bm{\mu}_{\mathcal{R}} - \bm{t}_{\mathrm{ref},\mathcal{R}}\right\|_2 
			\le \frac{\rho}{s_\pi} \sum_{i \in T^{(\delta)} } \pi_i = O(e^{-c_B\delta}). \qedhere
			\]
		\end{enumerate}
	\end{proof}
	
	\begin{lemma}[Bounds for weighted outputs]\label{lem:y_bound}
		Under Assumption~\ref{ass:out}, 
		\begin{enumerate}
			\item[(i)]\( \max_i|y^{(\delta)}_i-\mu_y|\le\Delta(\delta)+2B_y=O(\Delta(\delta)+1).
			\)
			\item[(ii)]Moreover, 
			\[
			\mu_y=\bar{y}_0^{(\delta)}+\frac{s_1}{s_\pi}\Delta(\delta)
			+O(\delta^{-1})+O\bigl((\Delta(\delta)+1)e^{-c_B\delta}\bigr).
			\]
		\end{enumerate}
	\end{lemma}
	\begin{proof}
		\begin{enumerate}
			\item[(i)] Under Assumption~\ref{ass:out}, 
			for $s\in\{0,1\}$,
			\[
			| y^{(\delta)}_i - \bar{y}^{(\delta)}_s | \leq B_y, \qquad i\in I^{(\delta)}_s,
			\]
			Hence, every $y^{(\delta)}_i$ lies in 
			\[[\bar{y}^{(\delta)}_0- B_y,\bar{y}^{(\delta)}_1 +B_y],\]
			the end points are up to swapping if needed.
			Since $\mu_y$ is a convex combination of the $y^{(\delta)}_i$'s, it also lies in this interval.
			Thus, 
			\[|y^{(\delta)}_i - \mu_y|\leq |\bar{y}^{(\delta)}_0- \bar{y}^{(\delta)}_1| +2B_y  = \Delta(\delta) + 2 B_y,\]
			and \[ \max|y^{(\delta)}_i - \mu_y| \leq \Delta(\delta) + 2 B_y.\] 
			\item[(ii)] For $i\in J_0^{(\delta)}$, Assumption~\ref{ass:out} gives $y^{(\delta)}_i=\bar{y}_0^{(\delta)}+O(\delta^{-1})$,
			and for $i\in J_1^{(\delta)}$, $y^{(\delta)}_i=\bar{y}_1^{(\delta)}+O(\delta^{-1})=\bar{y}_0^{(\delta)}+\Delta(\delta)+O(\delta^{-1})$.
			
			For $i\in T^{(\delta)}$, $|y^{(\delta)}_i-\bar{y}_s^{(\delta)}|\le B_y$ and $\pi_i=O(e^{-c_B\delta})$
			by Lemma~\ref{lem:BREAD_weights}(i).  
			
			Splitting the weighted sum and using $1/s_\pi=O(1)$,
			\[
			\mu_y=\bar{y}_0^{(\delta)}+\frac{s_1}{s_\pi}\Delta(\delta)
			+O(\delta^{-1})+O\bigl((\Delta(\delta)+1)e^{-c_B\delta}\bigr). \qedhere
			\]
		\end{enumerate}
		
	\end{proof}
	
	\subsubsection{Bounding \({b}_{\mathcal{R}}\)}
	We show that the $\mathcal{R}$-block of $\bm{b}$ is asymptotically negligible. Intuitively, this follows because the $\mathcal{R}$-features of the core points fluctuate only at order $O(\delta^{-2})$, while the non-core points are exponentially small weighted.
	\begin{lemma}[Upper bound on $\bm{b}_{\mathcal{R}}$]\label{lem:vR_B}
		Under Assumptions~\ref{ass:samdata} and~\ref{ass:out},
		\[
		\|\bm{b}_{\mathcal{R}}\|_2 = O\big((\Delta(\delta)+1)\delta^{-2}\big)+O\big((\Delta(\delta)+1)e^{-c_B\delta}\big), \qquad \delta\to\infty.
		\]
	\end{lemma}
	\begin{proof}
		Since
		\[\bm{b} = {(Z^{(\delta)})}^{\intercal} \Pi \bm{y} -\frac{1}{\sum_{i}^{n} \pi_i}{{(Z^{(\delta)})}^{ \intercal} \Pi \bm{1}}\bm{1}^{\intercal} \Pi \bm{y} 
		= \sum_{i=1}^{n} \pi_i (\bm{z}_i^{(\delta)} - \bm{\mu} )({y}^{(\delta)}_i - \mu_y),\]
		we have, 
		\[\| \bm{b}_{\mathcal{R}}\|_2 
		\le \max|{y}^{(\delta)}_i - \mu_y|\sum_{ i=1}^{n} \pi_i \|\bm{z}_{i,\mathcal{R}}^{(\delta)} - \bm{\mu}_{\mathcal{R}} \|_2.\]
		We bound \(\| \bm{b}_{\mathcal{R}}\|_2\) by splitting the sum into the tail set \(T^{(\delta)}\) and the two core sets
		\(J_0^{(\delta)}\cup J_1^{(\delta)}\). Lemma~\ref{lem:z_wm}(ii) and triangle inequality give, 
		\begin{enumerate}
			\item[(i)] for $i\in J_0^{(\delta)}\cup J_1^{(\delta)}$, 
			\[
			\max|{y}^{(\delta)}_i - \mu_y|\sum_{ i\in J_0^{(\delta)}\cup J_1^{(\delta)}} \pi_i \|\bm{z}_{i,\mathcal{R}}^{(\delta)} - \bm{\mu}_{\mathcal{R}}\|_2
			\leq \max|{y}^{(\delta)}_i - \mu_y|\sum_{ i\in J_0^{(\delta)}\cup J_1^{(\delta)}} \pi_i (O(\delta^{-2}) + O(e^{-c_B\delta})).
			\]
			\item[(ii)] for \(i\in T^{(\delta)}\),
			$\max|{y}^{(\delta)}_i - \mu_y| \sum_{i\in T^{(\delta)}} \pi_i = O(\max|{y}^{(\delta)}_i - \mu_y| e^{-c_B\delta})
			$
		\end{enumerate}
		
		By Lemma~\ref{lem:y_bound}, $\max_i|y^{(\delta)}_i-\mu_y|=O(\Delta(\delta)+1).$
		Therefore, 
		\[
		\| \bm{b}_{\mathcal{R}}\|_2 
		\leq
		O \big((\Delta(\delta)+1)e^{-c_B\delta}\big) + O \big((\Delta(\delta)+1)\delta^{-2}\big) \qedhere
		\]
	\end{proof}
	The preceding lemmas show that, under the BREAD kernel, the unshifted block $\bm{b}_{\mathcal{R}}$ is asymptotically negligible with order of $\delta^{-2}$ and up to exponentially small tail terms.
	
	\subsubsection{Bounding $M_{\mathcal{KR}}$}
	This section bounds the cross-block coupling $M_{\mathcal{KR}}$. This shows that, asymptotically, the BREAD linear system becomes nearly block diagonal.
	\begin{lemma}[Upper bound on $M_{\mathcal{KR}}$]\label{lem:MRK_B}
		Under Assumption~\ref{ass:samdata},
		\[
		\|M_{\mathcal{KR}}\|_2
		=
		\|M_{\mathcal{RK}}\|_2
		=
		O(\delta^{-1})+O(e^{-c_B\delta}),
		\qquad \delta\to\infty.
		\]
	\end{lemma}
	\begin{proof}
		Recall 
		$M = A +\lambda \bm{I}$ and
		$M_{\mathcal{R}\mathcal{K}} = A_{\mathcal{R}\mathcal{K}} = \sum_{i=1}^n \pi_i (\bm{z}_{i,\mathcal{R}}^{(\delta)} - \bm{\mu}_{\mathcal{R}} ) (\bm{z}_{i,\mathcal{K}}^{(\delta)} - \bm{\mu}_{\mathcal{K}} )^{\intercal}.$
		\[
		\|M_{\mathcal{KR}}\|_2
		=
		\|M_{\mathcal{RK}}\|_2
		=
		\left\|
		\sum_{i=1}^n \pi_i (\bm{z}_{i,\mathcal{R}}^{(\delta)} - \bm{\mu}_{\mathcal{R}} ) (\bm{z}_{i,\mathcal{K}}^{(\delta)} - \bm{\mu}_{\mathcal{K}} )^{\intercal} \right\|_2
		\leq 
		\sum_{i=1}^n \pi_i \|\bm{z}_{i,\mathcal{R}}^{(\delta)} - \bm{\mu}_{\mathcal{R}}\|_2 \|\bm{z}_{i,\mathcal{K}}^{(\delta)} - \bm{\mu}_{\mathcal{K}}\|_2
		\]
		
		By Lemma~\ref{lem:z_wm}(i), for all $i\in\{1,\dots,n\}$,
		\(\|\bm{z}_{i,\mathcal{K}}^{(\delta)}-\bm{\mu}_{\mathcal{K}}\|_2=O(\delta).\)
		
		Then, we split the sum into tail and core.
		
		For $i\in T^{(\delta)}$, by Lemmas~\ref{lem:z_wm}(i), $\|\bm{z}_{i,\mathcal{R}}^{(\delta)}-\bm{\mu}_{\mathcal{R}}\|_2=O(1)$ and by~\ref{lem:BREAD_weights}(i), $\pi_i\le K_B e^{-c_B\delta}$.
		Thus,
		\[
		\sum_{i\in T^{(\delta)}} \pi_i \|\bm{z}_{i,\mathcal{R}}^{(\delta)} - \bm{\mu}_{\mathcal{R}}\|_2 \|\bm{z}_{i,\mathcal{K}}^{(\delta)} - \bm{\mu}_{\mathcal{K}}\|_2
		= O(\delta e^{-c_B\delta}) = O(e^{-c^\prime_B\delta}),
		\]
		for some $0<c^\prime_B<c_B.$
		
		For $i\in J_0^{(\delta)}\cup J_1^{(\delta)}$, by Lemmas~\ref{lem:z_wm}(ii), $\|\bm{z}_{i,\mathcal{R}}^{(\delta)}-\bm{\mu}_{\mathcal{R}}\|_2=O(\delta^{-2})+O(e^{-c_B\delta}),$ and by~\ref{lem:BREAD_weights}(ii), $\pi_i \le 1$. Combined with \(\|\bm{z}_{i,\mathcal{K}}^{(\delta)}-\bm{\mu}_{\mathcal{K}}\|_2=O(\delta)\), it holds
		\[
		\sum_{i\in J_0^{(\delta)}\cup J_1^{(\delta)}} \pi_i \|\bm{z}_{i,\mathcal{R}}^{(\delta)} - \bm{\mu}_{\mathcal{R}}\|_2 \|\bm{z}_{i,\mathcal{K}}^{(\delta)} - \bm{\mu}_{\mathcal{K}}\|_2
		=
		O\left(\delta^{-1}\right) + O\bigl(\delta e^{-c^\prime_B\delta}\bigr).
		\]
		
		Hence,
		\begin{align*}
			\|M_{\mathcal{KR}}\|_2
			=
			\|M_{\mathcal{RK}}\|_2
			= 
			\left\|
			\sum_{i=1}^n \pi_i (\bm{z}_{i,\mathcal{R}}^{(\delta)} - \bm{\mu}_{\mathcal{R}} ) (\bm{z}_{i,\mathcal{K}}^{(\delta)} - \bm{\mu}_{\mathcal{K}} )^{\intercal} \right\|_2 
			& \leq  
			\sum_{i=1}^n \pi_i \|\bm{z}_{i,\mathcal{R}}^{(\delta)} - \bm{\mu}_{\mathcal{R}}\|_2 \|\bm{z}_{i,\mathcal{K}}^{(\delta)} - \bm{\mu}_{\mathcal{K}}\|_2 \\
			& =
			O\left(\delta^{-1}\right) + O\bigl(\delta e^{-c^\prime_B\delta}\bigr) \qedhere
		\end{align*} 
	\end{proof}
	Thus, the BREAD matrix $M$ is asymptotically block diagonal.
	
	\subsubsection{Lower bound on \(\hat{\bm{\beta}}_{\mathcal{K}}\)}
	It remains to show that \(\|\hat{\bm{\beta}}_{\mathcal{K}}\|_2\) does not vanish too fast. Write $\hat{\bm{\beta}}_{\mathcal{K}} = (S)^{-1} \left(\bm{b}_{\mathcal{K}} - M_{\mathcal{KR}} (M_{\mathcal{RR}})^{-1} \bm{b}_{\mathcal{R}}\right)$. The preceding Lemmas~\ref{lem:vR_B}, and~\ref{lem:MRK_B} imply that term $\left(M_{\mathcal{KR}} (M_{\mathcal{RR}})^{-1} \bm{b}_{\mathcal{R}}\right)$ is negligible. Thus, we first prove that the \(\bm{b}_{\mathcal{K}}\) has the lower bound with order $\delta\Delta(\delta)$, then combine this with the upper bound on $\|S\|_2$.
	\begin{lemma}[Lower bound on $\|\bm{b}_{\mathcal{K}}\|_2$]\label{lem:vK_B}
		Under Assumptions~\ref{ass:samdata} and~\ref{ass:out},
		\[
		\|\bm{b}_{\mathcal{K}}\|_2
		=
		\Omega(\delta\Delta(\delta)),
		\qquad \delta\to\infty.
		\]
		Since \(\Delta(\delta)\ge c_\Delta>0\), this implies
		\[
		\|\bm{b}_{\mathcal{K}}\|_2=\Omega(\delta).
		\]
	\end{lemma}
	\begin{proof}
		Recall $\bm{b}_{\mathcal{K}} = \sum_{i=1}^{n} \pi_i (\bm{z}_{i,\mathcal{K}}^{(\delta)} - \bm{\mu}_{\mathcal{K}})(y^{(\delta)}_{i} - \mu_y).$
		
		To find the lower bound on $\|\bm{b}_{\mathcal{K}}\|_2$, since $\|\bm{u}_{\mathcal{K}}\|_2 = 1$, we consider 
		\[
		\|\bm{b}_{\mathcal{K}}\|_2 \geq |\bm{u}_{\mathcal{K}}^{\intercal} \bm{b}_{\mathcal{K}}| = \left| \sum^{n}_{i=1} \pi_i (\bm{u}_{\mathcal{K}}^{\intercal} \bm{z}_{i,\mathcal{K}}^{(\delta)} - \mu_{\bm{u}})(y^{(\delta)}_{i} - \mu_y)\right|,
		\]
		where $\mu_{\bm{u}} = \frac{1}{s_\pi}\sum_{ i=1}^n \pi_{i} \bm{u}_{\mathcal{K}}^{\intercal} \bm{z}_{i,\mathcal{K}}^{(\delta)}.$

		We first estimate \({\mu}_{\bm{u}}\). Consider \(i\in J_0^{(\delta)}\), 
		\[ 
		\bm{u}_{\mathcal{K}}^{\intercal} \bm{z}_{i,\mathcal{K}}^{(\delta)} = \bm{u}_{\mathcal{K}}^{\intercal} \bm{t}_{\mathrm{ref},\mathcal{K}} + O(\delta^{-2}),
		\]
		while for \(i\in J_1^{(\delta)}\), since \(\bm{t}_{j^*}=\bm{t}_{\mathrm{ref}}+\delta \bm{u}\) and \(\|\bm{u}_{\mathcal{K}}\|_2=1\),
		\[ 
		\bm{u}_{\mathcal{K}}^{\intercal} \bm{z}_{i,\mathcal{K}}^{(\delta)} = \bm{u}_{\mathcal{K}}^{\intercal} \bm{t}_{\mathrm{ref},\mathcal{K}} + \delta + O(\delta^{-2}).
		\]
		For \(i\in T^{(\delta)}\), Assumption~\ref{ass:samdata} implies \(\bm{u}_{\mathcal{K}}^{\intercal} \bm{z}_{i,\mathcal{K}}^{(\delta)}=O(\delta)\), and Lemma~\ref{lem:BREAD_weights} gives \(\pi_i=O(e^{-c_B\delta})\). 
		
		Combining all,
		\[
		{\mu}_{\bm{u}} = \bm{u}_{\mathcal{K}}^{\intercal} \bm{t}_{\mathrm{ref},\mathcal{K}} + \frac{s_1}{s_\pi} \delta + O(\delta^{-2}) + O(\delta e^{-c_B\delta}).
		\]
		
		Then, we consider the $\|\bm{b}_{\mathcal{K}}\|_2$ for core points.
		
		For \(i\in J_0^{(\delta)}\),
		\[
		\bm{u}_{\mathcal{K}}^{\intercal} \bm{z}_{i,\mathcal{K}}^{(\delta)}- {\mu}_{\bm{u}}=-\frac{s_1}{s_\pi}\delta+O(\delta^{-2})+O(\delta e^{-c_B\delta}),
		\]
		and by Lemma~\ref{lem:y_bound}(ii),
		\[
		y^{(\delta)}_i-\mu_y=-\frac{s_1}{s_\pi}\Delta(\delta)+O(\delta^{-1})+O\bigl((\Delta(\delta)+1)e^{-c_B\delta}\bigr).
		\]
		Hence,
		\begin{align*}
			\bigl(\bm{u}_{\mathcal{K}}^{\intercal} \bm{z}_{i,\mathcal{K}}^{(\delta)} -  {\mu}_{\bm{u}}\bigr)\bigl(y^{(\delta)}_i-\mu_y\bigr)
			&=
			\frac{s_1^2}{s_\pi^2}\delta\Delta(\delta)
			+ O(1)
			+ O\!\bigl(\Delta(\delta)\delta^{-2}\bigr)
			+ O\!\bigl((\Delta(\delta)+1)\delta e^{-c_B\delta}\bigr).
		\end{align*}
		
		For \(i\in J_1^{(\delta)}\), similarly,
		\[
		\bm{u}_{\mathcal{K}}^{\intercal} \bm{z}_{i,\mathcal{K}}^{(\delta)}- {\mu}_{\bm{u}}=\frac{s_0}{s_\pi}\delta+O(\delta^{-2})+O(\delta e^{-c_B\delta}),
		\]
		\[y^{(\delta)}_i-\mu_y=\frac{s_0}{s_\pi}\Delta(\delta)+O(\delta^{-1})+O\bigl((\Delta(\delta)+1)e^{-c_B\delta}\bigr),
		\]
		and hence
		\begin{align*}
			\bigl(\bm{u}_{\mathcal{K}}^{\intercal} \bm{z}_{i,\mathcal{K}}^{(\delta)} -  {\mu}_{\bm{u}} \bigr)\bigl(y^{(\delta)}_i-\mu_y\bigr)
			&=
			\frac{s_0^2}{s_\pi^2}\delta\Delta(\delta)
			+ O(1)
			+ O\!\bigl(\Delta(\delta)\delta^{-2}\bigr)
			+ O\!\bigl((\Delta(\delta)+1)\delta e^{-c_B\delta}\bigr).
		\end{align*}
		
		Since Lemma~\ref{lem:BREAD_weights}(ii) gives \(\pi_i=1+O(\delta^{-1})\) on the core sets, summing over \(J_0^{(\delta)}\cup J_1^{(\delta)}\) yields
		\begin{align*}
			\sum_{i\in J_0^{(\delta)}\cup J_1^{(\delta)}} \pi_i
			\bigl(\bm{u}_{\mathcal{K}}^{\intercal} \bm{z}_{i,\mathcal{K}}^{(\delta)} -  {\mu}_{\bm{u}}\bigr)
			\bigl(y^{(\delta)}_i-\mu_y\bigr) & =
			\left[
			\frac{s_1^2}{s_\pi^2}\sum_{i\in J_0^{(\delta)}} \pi_i
			+
			\frac{s_0^2}{s_\pi^2}\sum_{i\in J_1^{(\delta)}} \pi_i
			\right]\delta\Delta(\delta) \\
			&\qquad\quad
			+ O(1)
			+ O\!\bigl(\Delta(\delta)\delta^{-2}\bigr)
			+ O\!\bigl((\Delta(\delta)+1)\delta e^{-c_B\delta}\bigr).
		\end{align*}
		
		By Lemma~\ref{lem:BREAD_weights}(ii) and~\ref{lem:weightedsum_r}, 
		since both $J_0$ and $J_1$ contribute terms of the same sign in \(\delta\Delta(\delta)\),
		the sum is bounded below by a positive multiple of \(\delta\Delta(\delta)\). 
		
		Therefore,
		\begin{align}
			\sum_{i\in J_0^{(\delta)}\cup J_1^{(\delta)}} \pi_i\bigl(\bm{u}_{\mathcal{K}}^{\intercal} \bm{z}_{i,\mathcal{K}}^{(\delta)} - {\mu}_{\bm{u}}\bigr) \bigl(y^{(\delta)}_i-\mu_y\bigr) = \Omega \bigl(\delta\Delta(\delta)\bigr).
			\label{eq:core_main_term}
		\end{align}
		
		Finally, we bound the tail part. 
		
		For \(i\in T^{(\delta)}\), by Lemma~\ref{lem:BREAD_weights}, \(\pi_i=O(e^{-c_B\delta})\), by Lemma~\ref{lem:y_bound}, $y^{(\delta)}_i-\mu_y = O(\Delta(\delta)+1)$, and Assumption~\ref{ass:samdata} implies \(\bm{u}_{\mathcal{K}}^{\intercal} \bm{z}_{i,\mathcal{K}}^{(\delta)} =O(\delta)\), and the bounds above yield \(\bm{u}_{\mathcal{K}}^{\intercal} \bm{z}_{i,\mathcal{K}}^{(\delta)} - {\mu}_{\bm{u}} = O(\delta).\)
		Hence
		\[
		\sum_{i\in T^{(\delta)}} \pi_i
		\left|
		\bigl(\bm{u}_{\mathcal{K}}^{\intercal} \bm{z}_{i,\mathcal{K}}^{(\delta)} - {\mu}_{\bm{u}}\bigr)
		\bigl(y^{(\delta)}_i-\mu_y\bigr)
		\right|
		=
		O \bigl((\Delta(\delta)+1)\delta e^{-c_B\delta}\bigr)
		=
		o \bigl(\delta\Delta(\delta)\bigr),
		\]
		because \(\Delta(\delta)\geq c_\Delta>0\).
		
		Combining this with \eqref{eq:core_main_term}, we obtain
		\[
		\bigl|\bm{u}_{\mathcal{K}}^{\intercal} \bm{b}_{\mathcal{K}}\bigr|=\Omega\!\bigl(\delta\Delta(\delta)\bigr).
		\]
		Therefore,
		\[
		\|\bm{b}_{\mathcal{K}}\|_2 \geq \bigl|\bm{u}_{\mathcal{K}}^{\intercal} \bm{b}_{\mathcal{K}}\bigr| = \Omega\!\bigl(\delta\Delta(\delta)\bigr).
		\]
		Since \(\Delta(\delta)\geq c_\Delta>0\), it follows that
		\[
		\|\bm{b}_{\mathcal{K}}\|_2=\Omega(\delta). \qedhere
		\]
	\end{proof}
	
	We now control the scale of the $\mathcal{K}$ block, reflecting the fact that the dominant variation in $A$ is along the shift direction.
	\begin{lemma}[Upper bound on $\|S\|_2$]\label{lem:s}
		
		Write the Schur complement $S = M_{\mathcal{KK}} - M_{\mathcal{KR}} (M_{\mathcal{RR}})^{-1}( M_{{\mathcal{KR}}})^\intercal$, 
		\[
		\|S\|_2 = O(\delta^2).
		\]
	\end{lemma}
	\begin{proof}
		Since \(M\) is symmetric positive definite, its Schur complement \(S = M_{\mathcal{KK}} - M_{\mathcal{KR}} (M_{\mathcal{RR}})^{-1} M_{\mathcal{RK}}\) is also symmetric positive definite. Moreover, $M_{\mathcal{KR}} (M_{\mathcal{RR}})^{-1} M_{\mathcal{RK}} \succeq 0.$ 
		Hence,
		\[
		0 \preceq S \preceq M_{\mathcal{KK}}.
		\]
		Therefore,
		\[\|S\|_2 \le \|M_{\mathcal{KK}}\|_2.\]
		
		Since $M_{\mathcal{KK}} = \lambda \bm{I}_{\mathcal{KK}} + A_{\mathcal{KK}},$
		where $A_{\mathcal{KK}} = \sum_{i=1}^{n} \pi_i (\bm{z}_{i,\mathcal{K}}^{(\delta)} - \bm{\mu}_{\mathcal{K}})(\bm{z}_{i,\mathcal{K}}^{(\delta)} - \bm{\mu}_{\mathcal{K}})^{\intercal}.$
		By the triangle inequality, 
		\[
		\|A_{\mathcal{KK}}\|_2 \leq \sum_{i=1}^{n} \pi_i \|\bm{z}_{i,\mathcal{K}}^{(\delta)} - \bm{\mu}_{\mathcal{K}}\|_2^2.
		\]
		By Lemma~\ref{lem:z_wm}(i),
		\[
		\|A_{\mathcal{KK}}\|_2 \leq \sum_{i=1}^{n} \pi_i O(\delta^2).
		\]
		Combined with $\pi_i\leq1$ and fixed $n$,
		\[
		\|S\|_2 \le \|M_{\mathcal{KK}}\|_2 = O(\delta^2). \qedhere
		\]
	\end{proof}
	\begin{lemma}[Lower bound on the BREAD coefficients on $\mathcal{K}$]\label{lem:betaK_B}
		Under Assumptions~\ref{ass:samdata} and~\ref{ass:out},
		\[
		\|\hat{\bm{\beta}}_{\mathcal{K}}\|_2
		=
		\Omega\!\left(\frac{\Delta(\delta)}{\delta}\right),
		\qquad \delta\to\infty.
		\]
		In particular,
		\[
		\|\hat{\bm{\beta}}_{\mathcal{K}}\|_2=\Omega(\delta^{-1}).
		\]
	\end{lemma}
	\begin{proof}
		Recall
		\[
		\hat{\bm{\beta}}_{\mathcal{K}}
		= (S)^{-1} \left(\bm{b}_{\mathcal{K}} - M_{\mathcal{KR}}(M_{\mathcal{RR}})^{-1}\bm{b}_{\mathcal{R}}\right).
		\]
		
		Thus, 
		\[
		S \hat{\bm{\beta}}_{\mathcal{K}} = \left(\bm{b}_{\mathcal{K}} - M_{\mathcal{KR}}(M_{\mathcal{RR}})^{-1}\bm{b}_{\mathcal{R}}\right).
		\]
		Then, we have,
		\[
		\|\left(\bm{b}_{\mathcal{K}} - M_{\mathcal{KR}}(M_{\mathcal{RR}})^{-1}\bm{b}_{\mathcal{R}}\right)\|_2 \leq
		\|S\|_2 \|\hat{\bm{\beta}}_{\mathcal{K}}\|_2.
		\]
		Hence,
		\[
		\|\hat{\bm{\beta}}_{\mathcal{K}}\|_2
		\ge \frac{1}{\|S\|_2}\Big(\|\bm{b}_{\mathcal{K}}\|_2
		- \|M_{\mathcal{KR}}\|_2 \|(M_{\mathcal{RR}})^{-1}\|_2 \|\bm{b}_{\mathcal{R}}\|_2 \Big).
		\]
		Since \(A\succeq 0\) and \(\lambda>0\), we have \(M\succeq \lambda \bm{I}\), hence
		\begin{equation}\label{eq:MRRinv}
			\|(M_{\mathcal{RR}})^{-1}\|_2 \le \frac{1}{\lambda}.
		\end{equation}
		
		Using Lemmas~\ref{lem:vR_B}, and~\ref{lem:MRK_B},  
		$\|M_{\mathcal{KR}}\|_2 = O(\delta^{-1})+O(e^{-c_B\delta})$, and $\|\bm{b}_{\mathcal{R}}\|_2
		=
		O \big((\Delta(\delta)+1)\delta^{-2}\big)
		+
		O \big((\Delta(\delta)+1)e^{-c_B\delta}\big).$
		Thus, the second term in parentheses is $O \big((\Delta(\delta)+1)\delta^{-3}\big) + O \big((\Delta(\delta)+1)e^{-c_B\delta}\big)$, hence negligible.
		
		Hence, by Lemmas~\ref{lem:vK_B} and~\ref{lem:s},
		\[
		\|\hat{\bm{\beta}}_{\mathcal{K}}\|_2
		\ge \frac{1}{\|S\|_2}\Big(\|\bm{b}_{\mathcal{K}}\|_2 - \|M_{\mathcal{KR}}\|_2\|(M_{\mathcal{RR}})^{-1}\|_2\|\bm{b}_{\mathcal{R}}\|_2\Big)
		\geq
		\frac{1}{O(\delta^2)}\big(\Omega(\delta \Delta(\delta))\big) = \Omega(\frac{\Delta(\delta)}{\delta}),\]
		
		Since, $\Delta(\delta)\geq c_\Delta >0$, in particular $\|\hat{\bm{\beta}}_{\mathcal{K}}\|_2 = \Omega(\delta^{-1}).$
	\end{proof}
	
	\subsubsection{Proof of Theorem~\ref{th:faith}}
	We now combine the previous bounds to obtain the asymptotic sparsity of \(\hat{\bm{\beta}}\).
	\begin{theorem}[Theorem~{3.3}, restated]
		Suppose the shift setup in (5) holds, the synthetic data matrix $Z^{(\delta)}$ satisfies Assumption~\ref{ass:samdata}, and the corresponding outputs, ${y}_i^{(\delta)} = f(\bm{z}_i^{(\delta)})$, satisfy Assumption~\ref{ass:out}.
		Let $\hat{\bm{\beta}}$ denote the estimator in (4) and define $\epsilon  = \dfrac{\| \hat{\bm{\beta}}_{\mathcal{R}} \|_{2} }{\| \hat{\bm{ \beta }}_{ \mathcal{K} } \|_{2}}$.
		Then,
		\[{\epsilon} \rightarrow 0, \qquad \text{as} \, \delta \rightarrow \infty.\]
	\end{theorem}
	\begin{proof}
		Recall
		\[
		\epsilon = \frac{\|\hat{\bm{\beta}}_{\mathcal{R}}\|_2}{\|\hat{\bm{\beta}}_{\mathcal{K}}\|_2}
		\le \frac{1}{\lambda}\left( \frac{\|\bm{b}_{\mathcal{R}}\|_2 }{\|\hat{\bm{\beta}}_{\mathcal{K}}\|_2} + \|M_{\mathcal{KR}}\|_2 \right).
		\]
		By Lemma~\ref{lem:vR_B},
		\(
		\|\bm{b}_{\mathcal{R}}\|_2 = O\!\bigl((\Delta(\delta)+1)\delta^{-2}\bigr) + O\!\bigl((\Delta(\delta)+1)e^{-c_B\delta}\bigr).
		\)
		By Lemma~\ref{lem:betaK_B},
		\(
		\|\hat{\bm{\beta}}_{\mathcal{K}}\|_2 = \Omega\!\left(\frac{\Delta(\delta)}{\delta}\right).
		\)
		Therefore,
		\begin{align*}
			\frac{\|\bm{b}_{\mathcal{R}}\|_2}{\|\hat{\bm{\beta}}_{\mathcal{K}}\|_2} = O\!\left(\frac{\Delta(\delta)+1}{\Delta(\delta)}\delta^{-1}\right)+O\!\left(\frac{\Delta(\delta)+1}{\Delta(\delta)}\delta e^{-c_B\delta}\right).
		\end{align*}
		Since \(\Delta(\delta)\ge c_\Delta>0\), we have
		\[
		\frac{\Delta(\delta)+1}{\Delta(\delta)}=O(1).
		\]
		Hence,
		\[
		\frac{\|\bm{b}_{\mathcal{R}}\|_2}{\|\hat{\bm{\beta}}_{\mathcal{K}}\|_2} = O(\delta^{-1})+O(\delta e^{-c_B\delta}).
		\]
		By Lemma~\ref{lem:MRK_B}, $\|M_{\mathcal{KR}}\|_2 = O(\delta^{-1})+O(e^{-c_B\delta}).$ 
		Combining all,
		\[
		\epsilon(\delta) \le \frac{1}{\lambda} \Bigl( O(\delta^{-1})+O(\delta e^{-c_B\delta}) + O(\delta^{-1})+O(e^{-c_B\delta}) \Bigr).
		\]
		Therefore,
		\[
		\epsilon(\delta) = O(\delta^{-1})+O(\delta e^{-c_B\delta}),
		\] 
		for some fixed \(0<c^\prime_B<c_B\).
		Thus,
		\[
		\epsilon(\delta)= O(\delta^{-1})+O(e^{-c_B'\delta}).
		\]
		In particular, \(\epsilon(\delta)\to 0\) as \(\delta\to\infty\).
	\end{proof}

	\section{Experiment Details}
	\subsection{Computation infrastructure}
	Experiments were run on a single compute node with an 192 cores AMD Genoa CPU, 336 GB RAM, and Red Hat Enterprise Linux 9.8 (Plow). We used Python 3.11.5, NumPy 1.26.4, scikit-learn 1.5.2, and SciPy 1.12.0.
	\subsection{Simulation Data Setup}\label{sec:simsetup}
	Specifically, we simulate $d=500$ features over $n=2000$ time points according to
	\begin{equation*}\label{eq:simdata}
		Y_t = \mu_0 + \phi Y_{t-1} 
		+ \left( a \cos\left(\frac{2\pi (t+\delta)}{T}\right)
		+ b \sin\left(\frac{2\pi (t+\delta)}{T}\right) \right)
		+ \varepsilon_t ,
	\end{equation*}
	where $Y_t \in \mathbb{R}^d$, $\mu_0 \in \mathbb{R}^d$ denotes the mean baseline vector, and $\phi$ is the AR(1) coefficient. The error term satisfies $\varepsilon_t \sim \mathcal{N}_d(0,\Sigma_0)$, where $\Sigma_0$ is the $d\times d$ in-control variance--covariance matrix. The final term is a first-order Fourier component that induces seasonality with period $T$, amplitude parameters $a$ and $b$, and phase shift $\delta$.
	\begin{figure}[h]
		\centering
		\includegraphics[width=0.25\columnwidth]{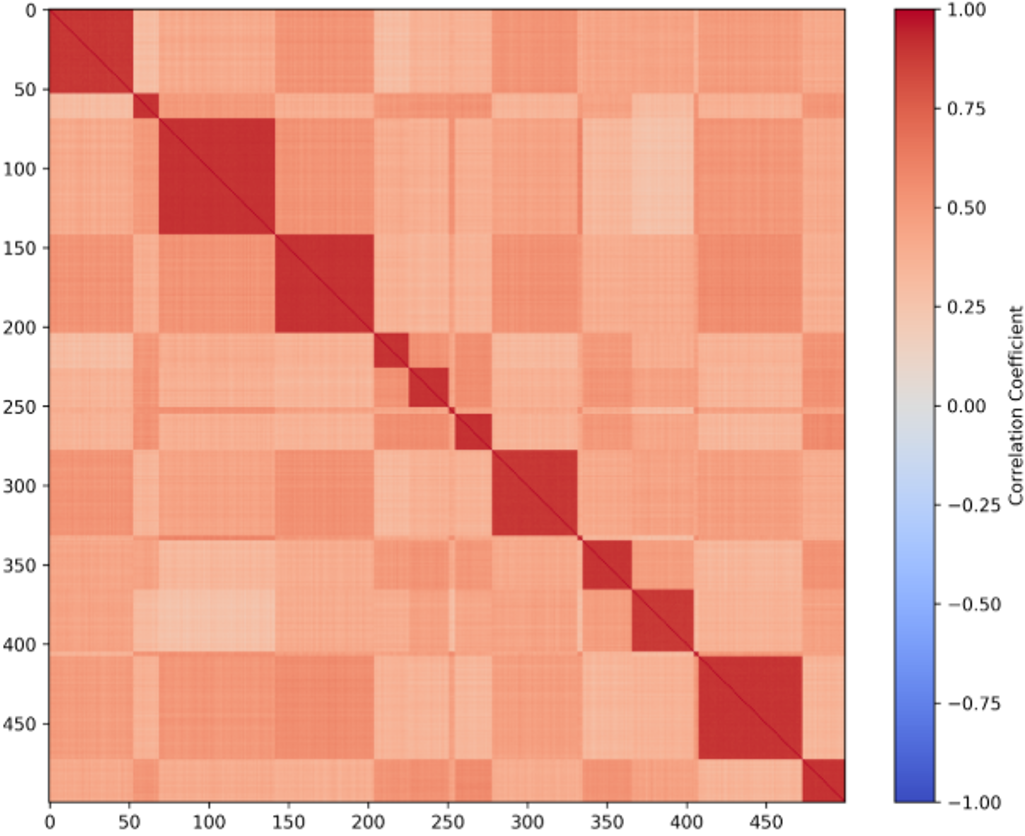}
		\caption{\textbf{Heatmap of the correlation matrix among the $d=500$ simulated features in Phase I.} 
			Variables are ordered by correlation blocks, yielding a clear block-diagonal pattern with strong within-block dependence ($>0.8$) and comparatively weaker between-block correlations (lighter off-diagonal regions, $<0.5$). The colour bar encodes Pearson correlation coefficients on $[-1,1]$, with diagonal elements equal to 1. There are 15 blocks in total with sizes of 53, 16, 73, 62, 22, 25, 4, 23, 54, 3, 31, 39, 3, 65 and 27, respectively.}
		\label{fig:cor_matrix}
	\end{figure}
	
	The covariance structure $\Sigma_0$ is constructed to exhibit block-wise dependence (i.e., correlation blocks) to reflect clustered feature relationships. Variables are ordered by correlation blocks, yielding a clear block-diagonal pattern with strong within-block dependence ($>0.8$) and comparatively weaker between-block correlations (lighter off-diagonal regions, $<0.5$). There are 15 blocks in total with sizes of 53, 16, 73, 62, 22, 25, 4, 23, 54, 3, 31, 39, 3, 65 and 27, respectively.
	An example of a correlation matrix is shown in Figure~\ref{fig:cor_matrix}.
\subsection{Monitoring Method Performance} \label{sec:monit_perf}
The monitoring performance of the AI-based methods is measured by false alarm probability (FAP), conditional expected delay (CED) and recall. 

FAP assesses Phase I performance, and all models are set to achieve a comparable FAP level. Under a pre-defined FAP, CED measures how quickly the method can detect the change, while recall rate denotes the fraction of signals detected in the total Phase II data. Both summarize Phase II detection performance, and lower FAP and CED indicate better performance.

For all scenarios we use, the FAP of the monitoring model is controlled as around $0.01$, and the averaged actual FAP over 1000 runs is about $0.0115$. CED and recall are reported in Table~\ref{table:ced&recall}.
\begin{table}[!ht]
	\centering
	\caption{Monitoring performance}
		\resizebox{\textwidth}{!}{
	\begin{tabular}{c|c|ccccccccccc}
		\hline
		Scenario & Metric & \multicolumn{11}{c}{Shift} \\
		\cline{3-13}
		& 
		& $-5$ & $-4$ & $-3$ & $-2$ & $-1$ & $0$ & $1$ & $2$ & $3$ & $4$ & $5$ \\
		\hline
		
		\multirow{2}{*}{S1}
		& CED   & 0.173 & 1.412 & 7.520 & 20.849 & 41.195 & 52.707 & 45.227 & 21.569 & 7.471 & 2.432 & 0.554 \\
		& Recall & 0.757 & 0.450 & 0.190 & 0.066  & 0.031  & 0.024  & 0.029  & 0.068  & 0.209 & 0.497 & 0.799 \\
		\hline
		
		\multirow{2}{*}{S2}
		& CED   & 11.524 & 17.593 & 25.196 & 36.243 & 46.105 & 52.707 & 47.773 & 39.409 & 29.002 & 22.353 & 17.081 \\
		& recal & 0.139  & 0.094  & 0.060  & 0.039  & 0.027  & 0.024  & 0.027  & 0.039  & 0.060  & 0.092  & 0.136 \\
		\hline
		
		\multirow{2}{*}{S3}
		& CED   & 0.015 & 0.267 & 2.343 & 10.401 & 32.264 & 52.707 & 38.090 & 15.011 & 7.047 & 3.217 & 1.007 \\
		& recal & 0.791 & 0.540 & 0.283 & 0.112  & 0.040  & 0.024  & 0.038  & 0.108  & 0.285 & 0.555 & 0.805 \\
		\hline
	\end{tabular}}
	\label{table:ced&recall}
\end{table}

\subsection{Hyperparameters}\label{sec:hyperparams}
Table~\ref{tab:hyperparams} summarises all hyperparameters used in the simulation study. All parameters are held fixed across methods unless otherwise stated. BREAD and LIME share the same ridge penalty $\lambda$, kernel width $\sigma$, and sample size $n$ to ensure a controlled comparison.

\begin{table}[h]
	\centering
	\caption{Hyperparameter settings.}
		\resizebox{\textwidth}{!}{
	\begin{tabular}{ccc}
		\hline
		Model & Parameter & Value\\
		\hline
		\multirow{5}{*}{{Surrogate model (shared by BREAD and LIME)}}
		& Ridge penalty              & $\lambda = 1$ 						   	\\
		& Surrogate model            & Weighted ridge regression               \\
		& Kernel type                & Gaussian                                \\
		& Kernel width               & $\sigma = 2.5$		    				\\
		& Sample size                & $n = 6000$                              \\
		\hline
		\multirow{4}{*}{{Anomaly detection model}}
		& Model type                 & LSTM-based $T^2$ control chart          \\
		& LSTM layers                & 3                  					\\
		& LSTM hidden dimension      & 48                  					\\
		& Training window length     & 168                						\\
		\hline
	\end{tabular}}
	\label{tab:hyperparams}
\end{table}

\subsection{Computational cost}\label{sec:compcost}
Table~\ref{tab:compcost} reports the average time per explanation, measured over the faithfulness experiments (1000 runs, Scenario 1, shift$=-5$). BREAD and LIME are comparable in runtime, as both require the same number of model evaluations with equivalent surrogate fitting cost. The naive method is approximately four times slower than the both, as it constructs a $502$ coalitions, including full and empty, and solves a $(502\times 501)$ weighted least-squares system, following a KernelSHAP-style leave-one-out strategy, which dominates the runtime at high dimension.
\begin{table}[!ht]
	\centering
	\caption{Run time}
	\begin{tabular}{c|ccc}
		\hline
		Model & BREAD	& LIME	& Naive	\\
		\hline
		Average Run time (std) & 0.459s (0.070)	& 0.460s (0.073) 	&1.962s (0.142)	\\
		\hline
	\end{tabular}
	\label{tab:compcost}
\end{table}

\subsection{Qualitative Results}
Figure~\ref{fig:qualitative_result} illustrates an example of diagnosing a signal by the proposed model and LIME under Scenario 1. 
In Figure~\ref{fig:qualitative_result_lime}, the feature importance is distributed nearly symmetrically around zero, which implies that the signal is driven by many features with similar contribution. By contrast, in Figure~\ref{fig:qualitative_result_blime}, there is a subset of features with notably high importance, indicating that these features contribute most to the monitoring statistic $T^2$. If we choose the top-53 features with the highest importance in this case, the proposed method is able to give the features with a shift.
\begin{figure}[!h]
	\centering
	\begin{subfigure}{0.4\columnwidth}
		\centering
		\includegraphics[width=\columnwidth]{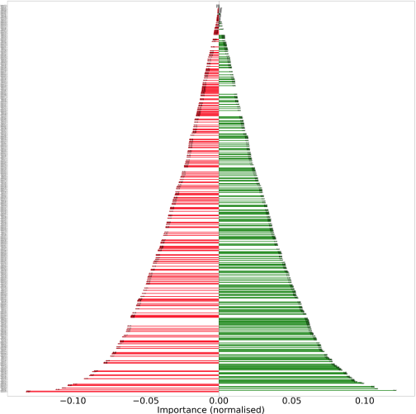}
		\caption{LIME}
		\label{fig:qualitative_result_lime}
	\end{subfigure}
	\begin{subfigure}{0.4\columnwidth}
		\centering
		\includegraphics[width=\columnwidth]{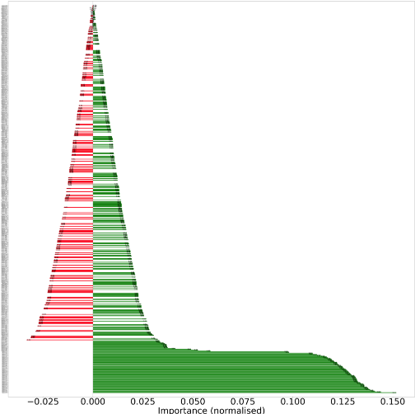}
		\caption{BREAD}
		\label{fig:qualitative_result_blime}
	\end{subfigure}
	\caption{\textbf{Signal diagnosis results.} Normalised feature importance is shown for two cases. Horizontal bars represent individual features ordered by importance. Values to the left of zero indicate negative contributions to the monitoring statistic, whereas values to the right indicate positive contributions. 
		(a) LIME. It shows an approximately symmetric distribution about zero. 
		(b) BREAD. It is positively skewed with an extended high-importance tail, overlapping with the actual shifted features (ground truth).}
	\label{fig:qualitative_result}
\end{figure} 

\subsection{Simulation Results}
The detailed simulation results are shown in Table~\ref{tab:acc}, which were also displayed in Figure~3 and~4.
\begin{table}[!h]
	\centering
	\caption{Explanation faithfulness under different scenarios and mean shift (Bold indicates the best performance).}
	\label{tab:acc}
	\setlength{\tabcolsep}{8pt}
	\resizebox{\textwidth}{!}{%
		\begin{tabular}{ccccccccc}
			\toprule
			&  & \multicolumn{4}{c}{Faithfulness (Cosine Similarity)} & \multicolumn{3}{c}{Robustness (Cosine Similarity)} \\
			\cmidrule(lr){3-6} \cmidrule(lr){7-9}
			Scenario & Shift & LIME & BREAD & Naive & Random pick & LIME & BREAD & Naive \\
			\midrule
			\multirow{11}{*}{S1}
			&$-5$ & $0.229\ (\pm 0.082)$ & $\bm{1.000}\ (\pm 0.011)$ & $0.694\ (\pm 0.111)$ & $0.106$ & $0.627\ (\pm 0.115)$ & $\bm{0.993}\ (\pm 0.036)$ & $0.698\ (\pm 0.106)$ \\
			&$-4$ & $0.201\ (\pm 0.084)$ & $\bm{0.990}\ (\pm 0.075)$ & $0.618\ (\pm 0.142)$ & $0.106$ & $0.588\ (\pm 0.128)$ & $\bm{0.958}\ (\pm 0.120)$ & $0.625\ (\pm 0.139)$ \\
			&$-3$ & $0.164\ (\pm 0.082)$ & $\bm{0.913}\ (\pm 0.219)$ & $0.492\ (\pm 0.183)$ & $0.106$ & $0.548\ (\pm 0.142)$ & $\bm{0.851}\ (\pm 0.239)$ & $0.524\ (\pm 0.173)$ \\
			&$-2$ & $0.131\ (\pm 0.077)$ & $\bm{0.663}\ (\pm 0.369)$ & $0.333\ (\pm 0.209)$ & $0.106$ & $0.490\ (\pm 0.154)$ & $\bm{0.550}\ (\pm 0.314)$ & $0.385\ (\pm 0.195)$ \\
			&$-1$ & $0.102\ (\pm 0.063)$ & $\bm{0.284}\ (\pm 0.338)$ & $0.167\ (\pm 0.171)$ & $0.106$ & $\bm{0.407}\ (\pm 0.150)$ & $0.183\ (\pm 0.170)$ & $0.245\ (\pm 0.169)$ \\
			&$0$  & $0.093\ (\pm 0.056)$ & $0.110\ (\pm 0.214)$ 	 & $0.091\ (\pm 0.118)$ & $0.106$ & $0.370\ (\pm 0.140)$ & $0.133\ (\pm 0.114)$ 	 & $0.203\ (\pm 0.142)$ \\
			&$1$  & $0.113\ (\pm 0.067)$ & $\bm{0.417}\ (\pm 0.388)$ & $0.219\ (\pm 0.197)$ & $0.106$ & $\bm{0.419}\ (\pm 0.146)$ & $0.226\ (\pm 0.183)$ & $0.264\ (\pm 0.163)$ \\
			&$2$  & $0.167\ (\pm 0.095)$ & $\bm{0.801}\ (\pm 0.338)$ & $0.452\ (\pm 0.233)$ & $0.106$ & $0.498\ (\pm 0.152)$ & $\bm{0.620}\ (\pm 0.307)$ & $0.413\ (\pm 0.191)$ \\
			&$3$  & $0.224\ (\pm 0.103)$ & $\bm{0.913}\ (\pm 0.209)$ & $0.609\ (\pm 0.198)$ & $0.106$ & $0.551\ (\pm 0.142)$ & $\bm{0.874}\ (\pm 0.223)$ & $0.552\ (\pm 0.170)$ \\
			&$4$  & $0.253\ (\pm 0.100)$ & $\bm{0.972}\ (\pm 0.088)$ & $0.679\ (\pm 0.162)$ & $0.106$ & $0.592\ (\pm 0.127)$ & $\bm{0.967}\ (\pm 0.104)$ & $0.653\ (\pm 0.134)$ \\
			&$5$  & $0.276\ (\pm 0.091)$ & $\bm{0.996}\ (\pm 0.020)$ & $0.726\ (\pm 0.125)$ & $0.106$ & $0.631\ (\pm 0.114)$ & $\bm{0.994}\ (\pm 0.030)$ & $0.724\ (\pm 0.102)$ \\
			\midrule
			
			\multirow{11}{*}{S2}
			&$-5$ & $0.152\ (\pm 0.055)$ & $\bm{0.998}\ (\pm 0.021)$ & $0.501\ (\pm 0.076)$ & $0.090$ & $0.540\ (\pm 0.131)$ & $\bm{0.989}\ (\pm 0.031)$ & $0.499\ (\pm 0.120)$ \\
			&$-4$ & $0.149\ (\pm 0.056)$ & $\bm{0.988}\ (\pm 0.069)$ & $0.461\ (\pm 0.089)$ & $0.090$ & $0.518\ (\pm 0.135)$ & $\bm{0.957}\ (\pm 0.080)$ & $0.464\ (\pm 0.132)$ \\
			&$-3$ & $0.144\ (\pm 0.057)$ & $\bm{0.917}\ (\pm 0.164)$ & $0.401\ (\pm 0.103)$ & $0.090$ & $0.489\ (\pm 0.140)$ & $\bm{0.810}\ (\pm 0.165)$ & $0.410\ (\pm 0.145)$ \\
			&$-2$ & $0.137\ (\pm 0.056)$ & $\bm{0.683}\ (\pm 0.260)$ & $0.305\ (\pm 0.124)$ & $0.090$ & $0.446\ (\pm 0.143)$ & $\bm{0.463}\ (\pm 0.184)$ & $0.327\ (\pm 0.155)$ \\
			&$-1$ & $0.129\ (\pm 0.054)$ & $\bm{0.310}\ (\pm 0.215)$ & $0.181\ (\pm 0.111)$ & $0.090$ & $\bm{0.387}\ (\pm 0.144)$ & $0.167\ (\pm 0.116)$ & $0.232\ (\pm 0.152)$ \\
			&$0$  & $0.123\ (\pm 0.053)$ & $0.089\ (\pm 0.049)$ 	 & $0.107\ (\pm 0.052)$ & $0.090$ & $0.355\ (\pm 0.142)$ & $0.115\ (\pm 0.107)$ & $0.187\ (\pm 0.138)$ \\
			&$1$  & $0.125\ (\pm 0.053)$ & $\bm{0.310}\ (\pm 0.193)$ & $0.171\ (\pm 0.095)$ & $0.090$ & $\bm{0.398}\ (\pm 0.140)$ & $0.169\ (\pm 0.111)$ & $0.238\ (\pm 0.142)$ \\
			&$2$  & $0.128\ (\pm 0.054)$ & $\bm{0.671}\ (\pm 0.241)$ & $0.283\ (\pm 0.109)$ & $0.090$ & $0.458\ (\pm 0.136)$ & $\bm{0.484}\ (\pm 0.165)$ & $0.337\ (\pm 0.144)$ \\
			&$3$  & $0.131\ (\pm 0.054)$ & $\bm{0.895}\ (\pm 0.181)$ & $0.370\ (\pm 0.099)$ & $0.090$ & $0.500\ (\pm 0.134)$ & $\bm{0.826}\ (\pm 0.140)$ & $0.413\ (\pm 0.138)$ \\
			&$4$  & $0.134\ (\pm 0.055)$ & $\bm{0.973}\ (\pm 0.095)$ & $0.429\ (\pm 0.092)$ & $0.090$ & $0.530\ (\pm 0.130)$ & $\bm{0.966}\ (\pm 0.066)$ & $0.466\ (\pm 0.127)$ \\
			&$5$  & $0.136\ (\pm 0.054)$ & $\bm{0.993}\ (\pm 0.035)$ & $0.468\ (\pm 0.088)$ & $0.090$ & $0.550\ (\pm 0.127)$ & $\bm{0.992}\ (\pm 0.024)$ & $0.501\ (\pm 0.119)$ \\
			\midrule
			\multirow{11}{*}{S3}
			&$-5$ & $0.255\ (\pm 0.049)$ & $\bm{1.000}\ (\pm 0.000)$ & $0.594\ (\pm 0.064)$ & $0.190$ & $0.674\ (\pm 0.095)$ & $\bm{0.991}\ (\pm 0.027)$ & $0.616\ (\pm 0.089)$ \\
			&$-4$ & $0.249\ (\pm 0.049)$ & $\bm{0.998}\ (\pm 0.010)$ & $0.564\ (\pm 0.069)$ & $0.190$ & $0.651\ (\pm 0.101)$ & $\bm{0.965}\ (\pm 0.070)$ & $0.592\ (\pm 0.101)$ \\
			&$-3$ & $0.242\ (\pm 0.050)$ & $\bm{0.968}\ (\pm 0.059)$ & $0.518\ (\pm 0.081)$ & $0.190$ & $0.627\ (\pm 0.108)$ & $\bm{0.892}\ (\pm 0.130)$ & $0.565\ (\pm 0.117)$ \\
			&$-2$ & $0.232\ (\pm 0.050)$ & $\bm{0.806}\ (\pm 0.185)$ & $0.437\ (\pm 0.103)$ & $0.190$ & $0.592\ (\pm 0.120)$ & $\bm{0.663}\ (\pm 0.191)$ & $0.514\ (\pm 0.142)$ \\
			&$-1$ & $0.216\ (\pm 0.052)$ & $\bm{0.443}\ (\pm 0.242)$ & $0.302\ (\pm 0.127)$ & $0.190$ & $\bm{0.515}\ (\pm 0.134)$ & $0.306\ (\pm 0.161)$ & $0.390\ (\pm 0.169)$ \\
			&$0$  & $0.203\ (\pm 0.049)$ & $0.190\ (\pm 0.151)$ 	 & $0.190\ (\pm 0.087)$ & $0.190$ & $0.437\ (\pm 0.128)$ & $0.217\ (\pm 0.122)$ & $0.283\ (\pm 0.156)$ \\
			&$1$  & $0.217\ (\pm 0.050)$ & $\bm{0.534}\ (\pm 0.244)$ & $0.324\ (\pm 0.130)$ & $0.190$ & $\bm{0.527}\ (\pm 0.128)$ & $0.336\ (\pm 0.161)$ & $0.407\ (\pm 0.155)$ \\
			&$2$  & $0.241\ (\pm 0.059)$ & $\bm{0.835}\ (\pm 0.182)$ & $0.472\ (\pm 0.116)$ & $0.190$ & $0.597\ (\pm 0.117)$ & $\bm{0.697}\ (\pm 0.172)$ & $0.523\ (\pm 0.134)$ \\
			&$3$  & $0.267\ (\pm 0.064)$ & $\bm{0.910}\ (\pm 0.148)$ & $0.546\ (\pm 0.105)$ & $0.190$ & $0.628\ (\pm 0.109)$ & $\bm{0.904}\ (\pm 0.117)$ & $0.575\ (\pm 0.115)$ \\
			&$4$  & $0.299\ (\pm 0.067)$ & $\bm{0.950}\ (\pm 0.072)$ & $0.589\ (\pm 0.083)$ & $0.190$ & $0.649\ (\pm 0.104)$ & $\bm{0.969}\ (\pm 0.064)$ & $0.602\ (\pm 0.100)$ \\
			&$5$  & $0.322\ (\pm 0.065)$ & $\bm{0.985}\ (\pm 0.026)$ & $0.615\ (\pm 0.069)$ & $0.190$ & $0.670\ (\pm 0.099)$ & $\bm{0.992}\ (\pm 0.024)$ & $0.625\ (\pm 0.088)$ \\
			\bottomrule
	\end{tabular}}
\end{table}
\newpage
\subsection{Case Study Data Visualisation} \label{sec:escdata}

\begin{figure}[!h]
	\centering
	\begin{subfigure}{0.24\columnwidth}
		\centering
		\includegraphics[width=\columnwidth]{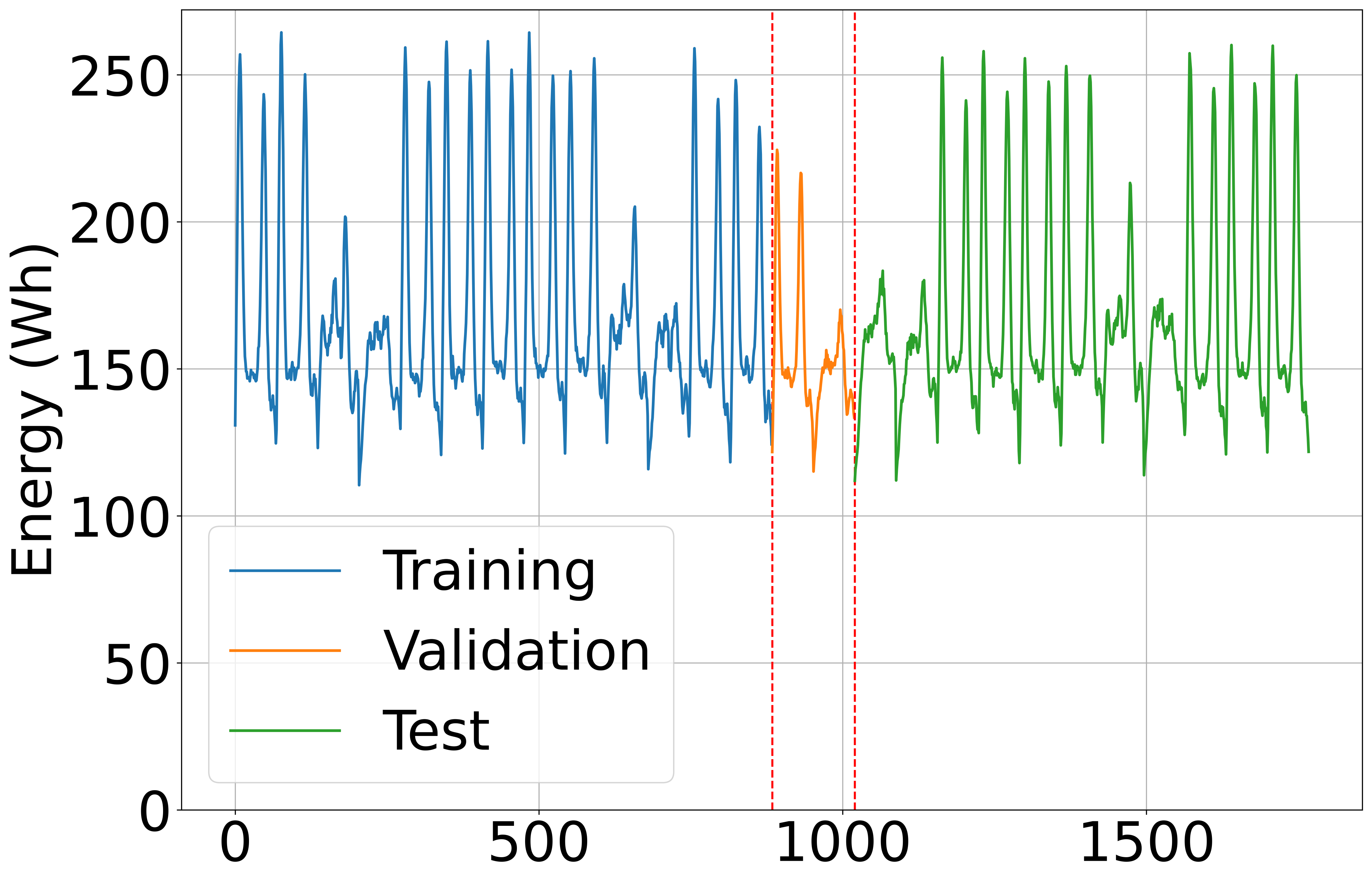}
		\caption{Escalator 00}
		\label{fig:esc_00}
	\end{subfigure}
	\hfill
	\begin{subfigure}{0.24\columnwidth}
		\centering
		\includegraphics[width=\columnwidth]{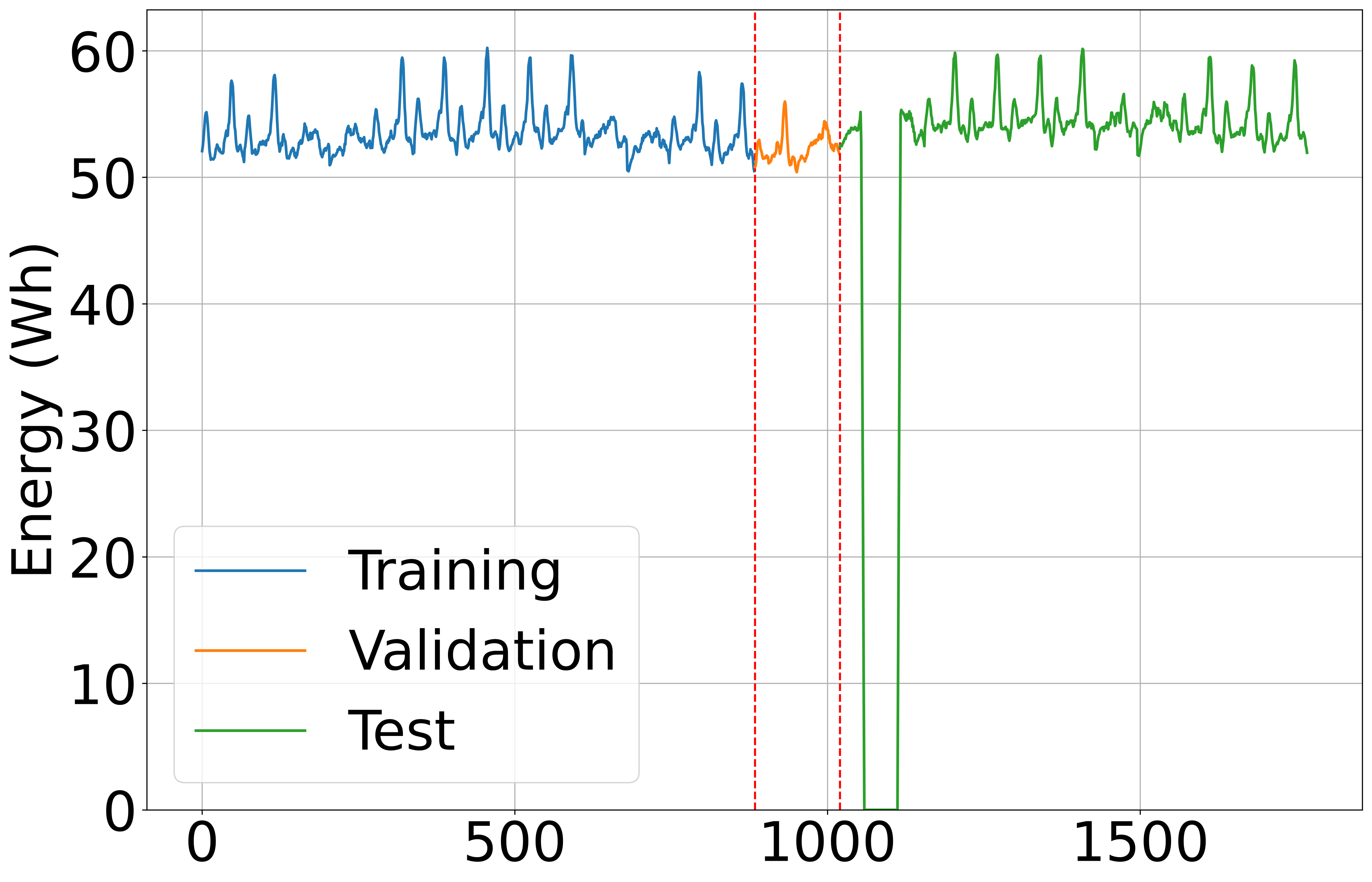}
		\caption{Escalator 08}
		\label{fig:esc_08}
	\end{subfigure}
	\hfill
	\begin{subfigure}{0.24\columnwidth}
		\centering
		\includegraphics[width=\columnwidth]{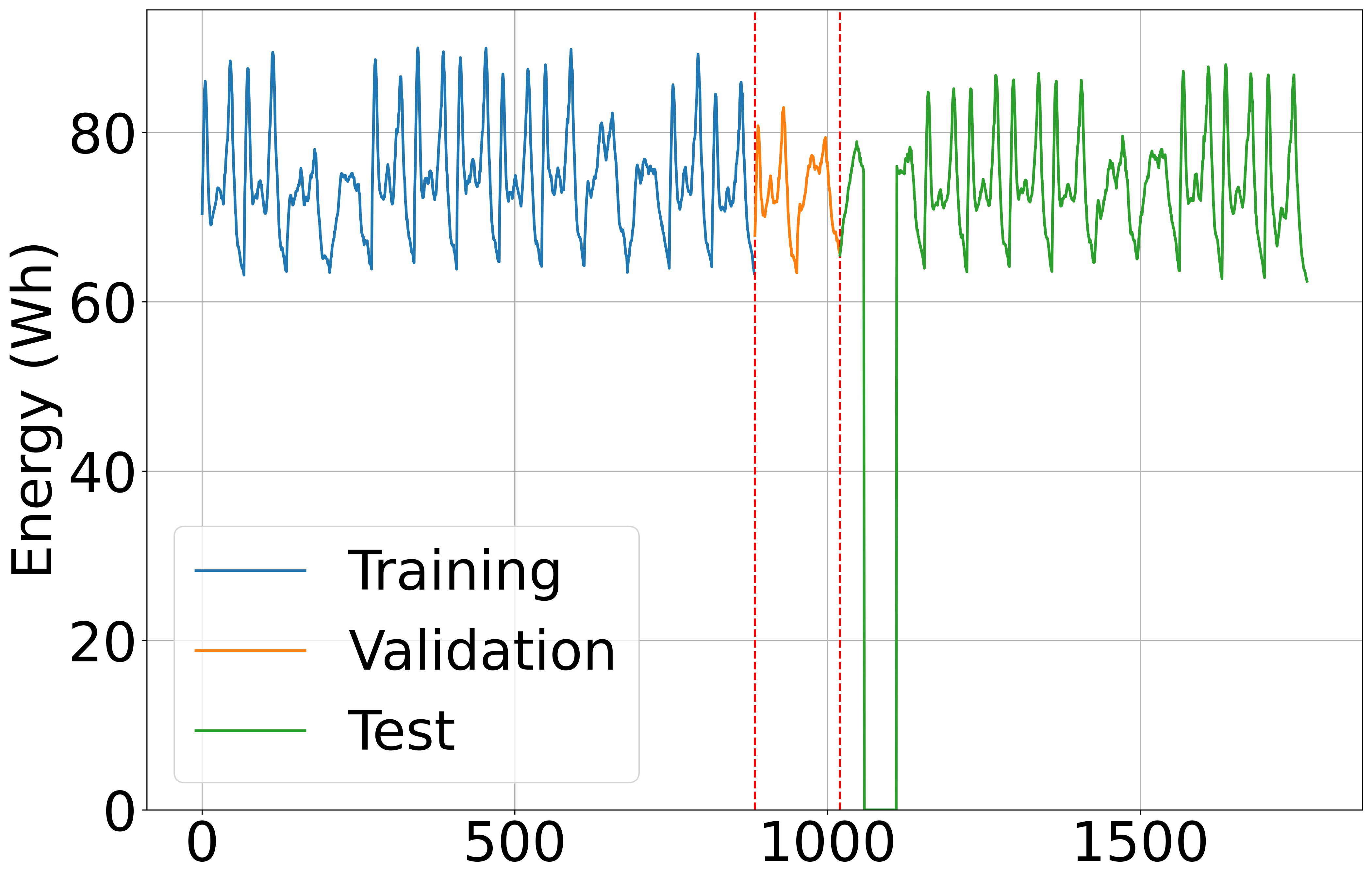}
		\caption{Escalator 13}
		\label{fig:esc_13}
	\end{subfigure}
	\hfill
	\begin{subfigure}{0.24\columnwidth}
		\centering
		\includegraphics[width=\columnwidth]{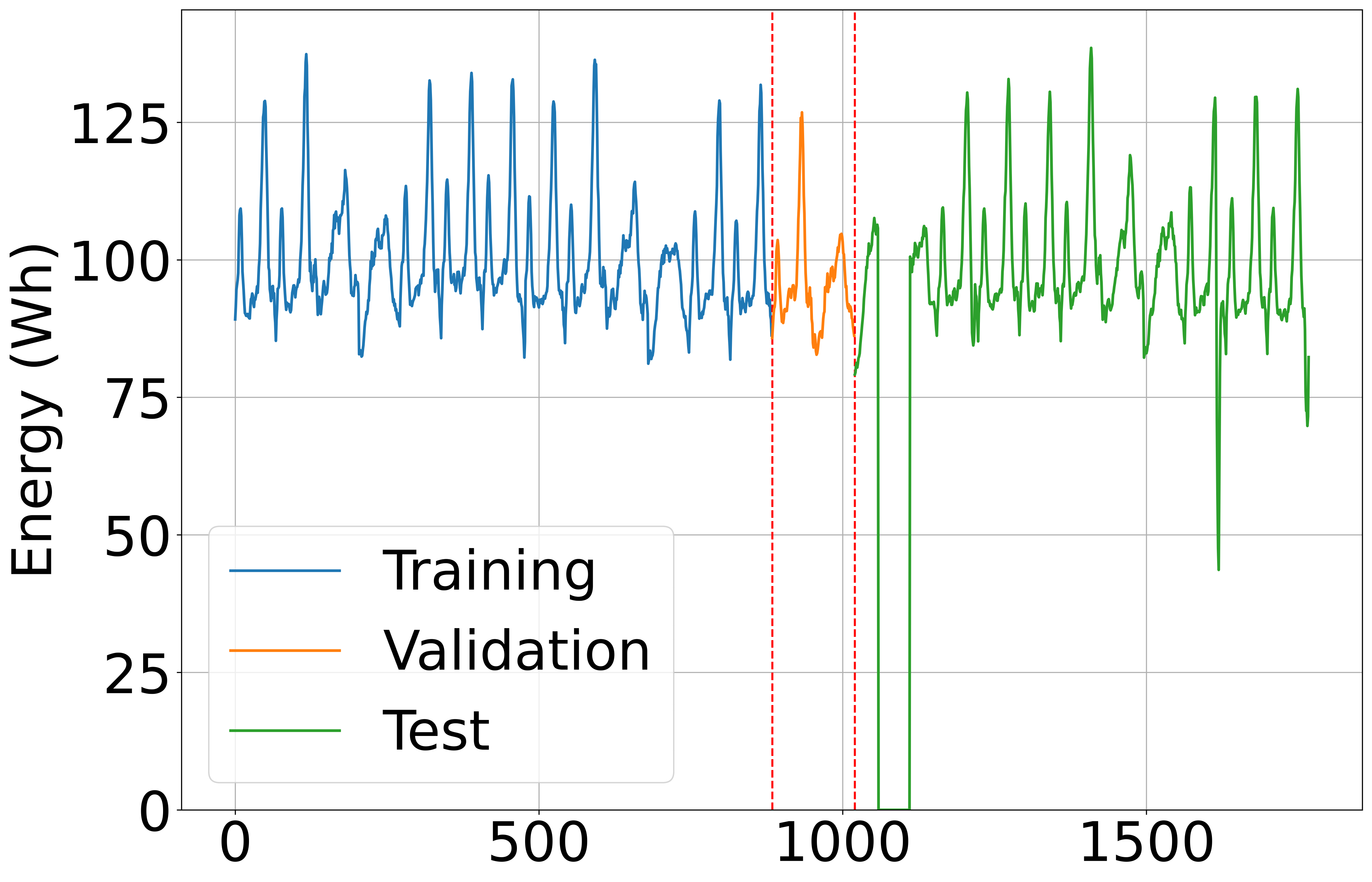}
		\caption{Escalator 14}
		\label{fig:esc_14}
	\end{subfigure}
	\caption{
		\textbf{Energy consumption data of selected monitored escalators across training, validation, and test data.}
		This figure presents the energy consumption patterns (Wh) of four escalators (Escalators 00, 08, 13, and 14) over time, segmented into training, validation, and test periods. The x-axis represents the time index, while the y-axis indicates energy consumption. Each subplot corresponds to a different escalator, highlighting distinct operational behaviours and energy profiles. 
	}
	\label{fig:data_casestudy}
\end{figure}

Here, we illustrate the energy data of Escalators 00, 08, 13, and 14. Escalator 00 represents the normal escalator behaviour, while escalators 08, 13, and 14 depict a period of non-energy reflecting abnormal operating conditions possibly caused by a fault or accidental shutdowns.
	\end{document}